\documentclass{article}
\usepackage{arxiv}
\usepackage[round]{natbib}

\usepackage{amsmath,amsfonts,amsthm,amssymb}
\usepackage{algorithm}
\usepackage[noend]{algorithmic}
\usepackage{array,booktabs,graphicx,subcaption,textcomp,xcolor}
\usepackage{tikz}
\usetikzlibrary{arrows.meta,positioning}
\usepackage{placeins}
\usepackage[hidelinks,hypertexnames=false]{hyperref}
\hypersetup{
  pdftitle={ProFL: Performative Robust Optimal Federated Learning},
  pdfauthor={Xue Zheng, Tian Xie, Xuwei Tan, Xueru Zhang, and Aylin Yener}
}

\newtheorem{theorem}{Theorem}
\newtheorem{lemma}{Lemma}
\newtheorem{corollary}{Corollary}
\newtheorem{assumption}{Assumption}
\newtheorem{remark}{Remark}

\renewcommand{\shorttitle}{}
\newsavebox{\proflfitbox}
\newcommand{\fitwidth}[3]{%
  \sbox{\proflfitbox}{#3}%
  \ifdim\wd\proflfitbox>#1\relax
    \resizebox{#1}{#2}{\usebox{\proflfitbox}}%
  \else
    \usebox{\proflfitbox}%
  \fi
}

\title{\textsc{ProFL}: Performative Robust Optimal \\Federated Learning}
\author{Xue Zheng, Tian Xie, Xuwei Tan, Xueru Zhang, and Aylin Yener\\
  The Ohio State University\\
  \texttt{zheng.1822@osu.edu, xie.1379@osu.edu, Tan.1206@osu.edu}\\
  \texttt{zhang.12807@osu.edu, yener@ece.osu.edu}%
  \thanks{This work was supported by the National Science Foundation under Grant 2112471.}}
\date{}

\begin{document}
\maketitle

{\small
\noindent Author accepted manuscript, \textit{IEEE Transactions on Signal Processing}.
\href{https://doi.org/10.1109/TSP.2026.3742668}{DOI: 10.1109/TSP.2026.3742668}.
\par
\noindent \copyright~2026 IEEE. Personal use of this material is permitted.
Permission from IEEE must be obtained for all other uses, in any current or future media,
including reprinting/republishing this material for advertising or promotional purposes,
creating new collective works, for resale or redistribution to servers or lists,
or reuse of any copyrighted component of this work in other works.\par}

\begin{abstract}
Performative prediction is a framework that captures distribution shifts that occur during the training of machine learning models due to their deployment. As the trained model is used, the induced data distribution evolves, leading to deviations from the original data distribution. The impact of such model-induced distribution shifts in federated learning (FL) is increasingly likely to transpire in real-life use cases. A recently proposed approach extends performative prediction to FL with the resulting model converging to a performative stable point, which may be far from the performative optimal point. Earlier research in centralized settings has shown that the performative optimal point can be achieved under model-induced distribution shifts, but these approaches require the performative risk to be convex and the training data to be noiseless. These assumptions are often violated in realistic FL systems. This paper addresses these limitations and proposes Performative Robust Optimal Federated Learning (\textsc{ProFL}), an algorithm that targets the performative optimal point in FL while mitigating the effects of noisy and contaminated local data.
We present a neighborhood-convergence analysis under the Polyak--{\L}ojasiewicz condition, which applies to non-convex objectives. Extensive experiments on multiple datasets demonstrate the advantage of Robust Optimal Federated Learning over the state-of-the-art.
\end{abstract}


\section{Introduction}\label{sec:intro}

Federated learning (FL) is a distributed learning paradigm that enables efficient training across decentralized clients while preserving data privacy. This privacy preservation is achieved by keeping the raw training data on each client and only exchanging model updates with the central server. Such a framework has attracted significant attention in privacy-sensitive applications such as the Internet of Things, healthcare, and financial services.

Conventional FL (e.g., collaborative model training via FedAvg) typically assumes that each client's data distribution remains static throughout the training process or between the training and deployment phases. However, this assumption often fails in real-world applications, where user behavior and data characteristics are continuously evolving. Although some previous work considers ``dynamics'' in FL, they largely focus on changes across clients (i.e., statistical heterogeneity) or time-varying participation rates~\citep{rizk2020dynamic,park2021tackling,wang2022unified,zhu2021diurnal}, rather than on distributional shifts within each client's local data over time.

Data shifts are common in FL deployment. For instance, periodic deployment and retraining often lead to a mismatch between the data used for training and the data encountered in the future. Recent work begins to address such settings by training on time-evolving data~\citep{guo2021towards,ma2022continual}, or by training on one distribution and deploying on another~\citep{nguyen2022fedsr,jiang2022test}. However, these works primarily consider \textit{exogenous} shifts---distribution changes that occur independently of the model, such as variations in sensor quality, lighting, or environmental conditions.

In contrast, many practical FL applications involve \textit{endogenous} distribution shifts, where the deployed model itself induces changes in the data distribution. For example, recommendation systems shape user preferences~\citep{dean2022}, navigation apps influence traffic flow, and large language models affect language usage~\citep{anderson2025modelmisalignmentlanguagechange}. Users may also adapt their behavior to exploit automated decision systems, such as those in hiring or credit scoring~\citep{Hardt2016a,zhang2022}. These model-driven shifts fall under the framework of \textit{performative prediction} (PP), which seeks performative \textit{stable}~\citep{perdomo2020performative,shan2023beyond} or \textit{optimal}~\citep{pmlr-v139-izzo21a,miller2021outside} solutions. {These PP studies focus on centralized settings, and recent work has begun to study PP in distributed settings, including FL and decentralized learning settings~\citep{jin2023performative,NEURIPS2022_fad7c708}. However, these PP works mainly focus on convergence to \textit{performative stable} (PS) points under clean local data. They do not address how to find a \textit{performative optimal} (PO) point when local data are heterogeneous and potentially contaminated.}

Extending PP to the FL setting is highly nontrivial.
While centralized approaches can rely on access to a global dataset and
full observability of distribution shifts, FL systems must contend with
decentralized and partially observable dynamics. Moreover, the interaction
between distributed optimization and performative feedback can lead to
suboptimal convergence or even divergence if not properly addressed. {This difficulty is more pronounced when finding a PO point, because optimizing the \emph{performative risk} requires estimating the \textit{performative gradient}. Unlike the ordinary empirical gradient, the performative gradient contains a distribution-dependent component that captures how the induced data distribution changes with the deployed model. In practice, this component is estimated from local samples using \textit{empirical distributional statistics}, such as sample means, variances, or other low-dimensional summaries, which is sensitive to the data contamination.}

In realistic FL deployment, these local estimates may be unreliable due to exogenous perturbations and sample-level contamination. For example, edge devices may produce corrupted records due to sensor failures, communication errors, data entry errors, abnormal environmental conditions, or malicious manipulation. Such contaminated samples can significantly bias empirical statistics, thereby distorting the estimated distribution-dependent component of the performative gradient. This creates a direct coupling between endogenous performative feedback and local data contamination. In addition, unlike in a centralized setting, raw data in an FL setting cannot be pooled at the server for global inspection, preprocessing, or manual cleaning. Therefore, each client must estimate its local distributional response using limited and potentially contaminated samples. Moreover, clients may have heterogeneous data distributions and different contamination patterns, making it insufficient to simply apply centralized performative-gradient methods or standard FL algorithms. Therefore, a principled federated algorithm is needed to find the PO solution under both model-induced endogenous distribution shifts and exogenous data contamination.

\subsection{Related Work}

Performative prediction~\citep{perdomo2020performative} was first introduced in 2020 as an optimization framework to deal with endogenous data distribution shifts, where the model deployed to the environment affected the subsequent data, resulting in the collected data changing as the deployed model changed. Common applications of performative prediction included strategic classification~\citep{Hardt2016a,xie2024nonlinear} and predictive policing~\citep{ensign2018runaway}. Early work focused primarily on finding PS points. For example, \citet{perdomo2020performative} proposed \textit{Repeated Risk Minimization} to find the PS point and bounded the distance between the PS point and the PO  point. \citet{mendler2022anticipating} designed the first algorithm to find the performative stable point under the online learning setting. \citet{mendler2020stochastic} proved convergence of greedy and lazy deployment under smoothness and strong convexity.

Subsequent work moved toward directly identifying PO  point. \citep{miller2021outside} optimized performative risk directly, although the scope of the distribution map was restricted, and \citep{pmlr-v139-izzo21a} designed an algorithm, performative gradient descent, for finding the performative optimal solution. \citep{Zhao2022OptimizingTP} relaxed the assumption of strong convexity to the weakly-strongly convex case and proved the convergence of the deployment. \citep{somerstep2024learning,bracale2024learning} focused on learning the distribution map in performative prediction and proposed a framework named reverse causal PP. In addition, PP was also related to reinforcement learning~\citep{pmlr-v162-zheng22c} and bandit problems~\citep{chen2023performative}.

Few recent works have examined PP in distributed learning, including FL~\citep{jin2023performative} and decentralized learning~\citep{NEURIPS2022_fad7c708}. These works analyzed PS points by establishing their existence, proving convergence to such stable points, and quantifying their distance to the PO point. Among these works, \citep{jin2023performative} is the closest to ours because it studies PP in FL. However, they mainly focus on finding PS points rather than the PO point. This distinction is important because the PS point can be far from the PO point. Thus, finding the PO point in distributed learning settings remains necessary. Moreover, existing distributed PP works do not study robustness to exogenous perturbations or sample-level local data contamination. They also do not explicitly address the impact of limited local samples on performative-gradient estimation, which is especially important in FL. In contrast, our work targets the PO point in FL and incorporates local robust-gradient estimation, server-side aggregation, and adaptive sample-size selection to mitigate the effects of heterogeneous clients, limited local samples, and high-impact contaminated gradients.

\subsection{Contributions}
In this paper, we address these challenges by proposing
\textbf{P}erformative \textbf{r}obust \textbf{o}ptimal
\textbf{F}ederated \textbf{L}earning (\textsc{ProFL}), an FL algorithm
that targets a PO solution under
endogenous distribution shifts while mitigating high-impact local data
contamination. We provide a theoretical analysis of data contamination
and the convergence of \textsc{ProFL}. Compared to the federated
extension of~\citep{pmlr-v139-izzo21a}, \textsc{ProFL} employs simple yet
effective strategies and provides convergence guarantees
under the PL condition, covering a larger class of potentially
non-convex objective functions. Our key contributions compared to prior
works are:
\begin{itemize}
    \item \textbf{Performative optimal point.}
    Unlike~\citep{jin2023performative}, which only guarantees convergence
    to a stable point, \textsc{ProFL}
    directly optimizes the performative risk toward the PO
    point.

    \item \textbf{Convergence analysis for a larger class of objectives.}
    Unlike prior works~\citep{pmlr-v139-izzo21a,jin2023performative} that
    show convergence for convex performative risks, we
    establish convergence to a neighborhood of the PO point
    under the non-convex Polyak--{\L}ojasiewicz (PL) condition.

    \item \textbf{Robustness to distribution shifts and data contamination.}
    \textsc{ProFL} incorporates local robust-gradient
    estimation to mitigate contamination-induced distortion of the
    performative-gradient estimate. The robust-gradient mechanism targets
    high-impact samples that exhibit a directional high-score pattern,
    while the convergence analysis explicitly retains the effect of any
    residual contamination.
\end{itemize}

\section{System Model}\label{sec:pf}
\subsection{Performative Prediction Model}
\subsubsection{Centralized Performative Prediction Model}

\begin{figure}[t]
\centering
\scriptsize
\begin{tikzpicture}[node distance=0.6cm, >=Stealth, thick]

\node[draw, rectangle, minimum height=0.9cm, minimum width=1.5cm] (user) {User};
\node[draw, rectangle, minimum height=0.9cm, minimum width=1.7cm, right=of user] (data) {Data \( D(\theta^t) \)};
\node[draw, rectangle, minimum height=0.9cm, minimum width=1.5cm, right=of data] (model) {Model \( \theta^{t+1} \)};
\node[draw, rectangle, minimum height=0.9cm, minimum width=1.7cm, right=of model] (pred) {Predictions};

\draw[->] (user) -- (data);
\draw[->] (data) -- (model);
\draw[->] (model) -- (pred);

\draw[->, thick] (pred.south)
  .. controls +(down:1.2cm) and +(down:1.2cm) ..
  (user.south)
  node[midway, below, align=center, text=black]
  {\scriptsize  Predictions affect behavior, which shifts data distribution};

\end{tikzpicture}
\caption{System model of performative prediction. The model \( \theta \) influences the data distribution \( D(\theta) \), forming a feedback loop that drives endogenous data shifts.}
\label{fig:performative_model}
\end{figure}

Performative prediction addresses settings where the deployed model influences future data distributions (see Fig.~\ref{fig:performative_model}). Let $\theta \in \Theta \subseteq \mathbb{R}^d$ denote the model parameters, where $\Theta$ is closed and convex and $d$ is the dimension. The data distribution is \( {D}(\theta) \) which depends on \( \theta \) and a distribution map \( {D}(\cdot) \). $z \in \mathcal{Z}$ denotes a sample from the sample space $\mathcal{Z}$. Denote the loss function with model \( \theta \) and data sample \( z \) as \( \ell(z; \theta) \). PP considers two risk formulations. The performative risk (PR)is defined as
\[
PR(\theta) := \mathbb{E}_{z \sim {D}(\theta)} [\ell(z; \theta)].
\]
This is called performative risk because, instead of measuring the risk over a fixed distribution, it measures the expected loss of model $\theta$ on the data distribution $D(\theta)$ it induces.
While the decoupled performative risk (DPR) is given by
\[
DPR(\theta, \theta') := \mathbb{E}_{z \sim {D}(\theta)} [\ell(z; \theta')].
\]
Unlike PR, where $D(\theta)$ depends on $\theta$, the DPR decouples the two. Based on the definitions of DPR and PR, we define two key concepts: the PS and the PO points. The PS point \( \theta^{\text{PS}} \) is a fixed point that captures an equilibrium state and is defined as:
\begin{align}
    \theta^{\text{PS}} := \arg\min_{\theta} \mathbb{E}_{z \sim {D}(\theta^{\text{PS}})} [\ell(z;\theta)]. \label{eq:PS_global}
\end{align}
The data distribution is induced by ${\theta}^{\text{PS}}$, while $\theta$ is optimized. As ${\theta}^{\text{PS}}$ is the fixed point of Eq.~\eqref{eq:PS_global}, it can be found by \textit{repeated risk minimization} (RRM) or \textit{repeated gradient descent} (RGD) \citep{perdomo2020performative}. RRM recursively finds a model $\theta^{t+1}$ that minimizes DPR on the distribution $D(\theta^{t})$ induced by the previous model, whereas RGD recursively finds a model $\theta^{t+1}$ through gradient descent, and $\theta^{t+1}$ is not necessarily the minimizer of the DPR on $D(\theta^{t})$. The performative optimal (PO) point, \( \theta^{\text{PO}} \), minimizes the performative risk (PR) and is defined as:
\[
\theta^{\text{PO}} := \arg\min_{\theta} \mathbb{E}_{z \sim {D}(\theta)} [\ell(z; \theta)].
\]

 The key to finding the PO point is to update the model using the actual gradient of PR, $\nabla \mathcal{L}(\theta)$, also known as the performative gradient. Assume the performative distribution \( D(\theta) \) has a density function \( p(z; f(\theta)) \), where the parameter \( f(\theta) \) is a differentiable function of \( \theta \). Then, the performative risk can be written as
\[
\mathcal{L}(\theta) = \int \ell(z; \theta) \, p(z; f(\theta)) \, dz,
\]
provided that \( p \) and \( f \) are continuously differentiable.\footnote{An example of \( D(\theta) \) is a mixture of Gaussians, \( \sum_{k=1}^{K} \nu_k \mathcal{N}(f_k(\theta), \sigma_k^2) \), where the mean depends on the model \citep{pmlr-v139-izzo21a}. This can approximate any smooth density to arbitrary precision.}

 We now present the formal problem setting of federated performative prediction, where the objective is to find the global performative optimal point in the presence of heterogeneous data shifts and possible data contamination.

\subsubsection{Federated Performative Prediction Model}
 Consider a set of clients $\mathcal{V} = \{1, \ldots , N\}$ collaboratively training an FL model $\theta$ without sharing their local raw data. Such model-induced data distributions, introduced in PP \citep{perdomo2020performative}, are characterized by a \textit{performative distribution map} $D_i: \Theta \to \mathcal{P}(\mathcal{Z})$. Specifically, deploying model $\theta$ at client $i$ results in data following distribution $D_i(\theta)$. The distribution map \( D_i(\cdot) \) can vary across clients, resulting in data heterogeneity in the federated learning system. The goal is to train a global model that minimizes the total loss across all clients:
\begin{equation}\label{eq:obj}  {{\theta}}^{\text{PO}} \in \mathop{\arg\min}\limits_{\theta} \mathcal{L}(\theta)
\end{equation}
where $\mathcal{L}(\theta)
:=
\sum_{i\in\mathcal{V}}\alpha_i\mathcal{L}_i(\theta),\mathcal{L}_i(\theta) = \mathbb{E}_{Z_i \sim D_i(\theta)}[\ell(Z_i;\theta)]$ is the PR of client $i$ for a given loss function $\ell(Z_i;\theta)$ on local data $Z_i\sim D_i(\theta)$. Here, $\alpha_i \geq 0$ represents the fraction of client $i$'s data relative to the total, with $\sum_{i \in \mathcal{V}} \alpha_i = 1$. In the FL setting, the {global performative risk over all clients} is denoted by
$\mathcal{L}(\theta)$, and the performative optimal point
$\theta^{\mathrm{PO}}$ is defined as its minimizer.

Since the data distribution $D_i(\theta)$ depends on $\theta$, which is the variable being optimized, finding the PO point is often challenging. Thus, existing works \citep{perdomo2020performative,mendler2020stochastic,Zhao2022OptimizingTP} have mostly focused on finding the PS point. In an FL system, the PS point ${{\theta}}^{\text{PS}}$ is defined as follows:

\begin{equation}\label{eq:PS}
\theta^{\mathrm{PS}}
\in
\mathop{\arg\min}\limits_{\theta\in\Theta}
\mathcal{L}(\theta^{\mathrm{PS}};\theta)
\end{equation}
where $\mathcal{L}(\theta';\theta)
:=
\sum_{i\in\mathcal{V}}
\alpha_i\mathcal{L}_i(\theta';\theta)$ and $ \mathcal{L}_i({\theta}'; \theta) = \mathbb{E}_{Z_i \sim D_i({\theta}')}[\ell(Z_i;\theta)]$ denotes client $i$'s DPR.

\subsection{Data Contamination Model}
In realistic FL systems, each client's data distribution
$D_i(\theta)$ is unknown, and the training data may be noisy, corrupted,
or maliciously manipulated. In this work, we consider a general
contamination model based on Huber's $\epsilon$-contamination
framework~\citep{huber1992robust}, which serves as a foundation for
robust statistical analysis in the presence of outliers and data
corruption. This model assumes that besides the data distribution we aim
to learn, a small proportion of samples can come from an arbitrary
distribution. This contamination framework is crucial for creating
machine learning algorithms robust to outliers, thereby ensuring more
reliable analysis in practical scenarios where data imperfections are
common. Huber's model has been fundamental in the field of robust
statistics, influencing a wide range of applications and subsequent
research. To account for this, we assume each client $i$ has a fixed fraction $\epsilon_i^{\mathrm{orig}} \in [0, 1]$ of data samples from another \textit{fixed but arbitrary} distribution $Q_i$. Thus, the data acquired from client $i$ follows a mixture distribution $
Z_i\sim P_i(\theta)
=
(1-\epsilon_i^{\mathrm{orig}})D_i(\theta)
+
\epsilon_i^{\mathrm{orig}}Q_i.
$ We term $Q_i$ data ``contamination''.

\subsection{Research Objectives and Challenges}
Our goal is to design FL algorithms that
optimize the performative risk toward
$\theta^{\text{PO}}$ in the presence of contaminated data
$Z_i\sim P_i(\theta)$ and characterize the effect of contamination on
convergence. Importantly, because $Q_i$ is arbitrary and unknown, our objective is
not to identify and remove every sample drawn from $Q_i$. Instead,
\textsc{ProFL} targets high-impact contaminated samples whose gradients
1) have a substantial effect on the model update by systematically
deviating along a dominant direction, and 2) form a separated
high-score tail in the centered gradient distribution. While data contamination has been studied in FL with static data, its
effect in performative settings is unexplored. When client data depends
on the FL system, contamination in one client may cause a cascading
effect across the network, making it crucial to investigate how this
disrupts FL convergence.

\begin{remark}
FL finds the solution by alternating between \textit{global model
aggregation} at the central server and \textit{local model updates} at
the clients. In our proposed algorithms, we focus on the local updates,
which work with different aggregation rules imposed by the server, such
as differing weighted averages.
\end{remark}

\section{Algorithm Design} \label{sec:algo}

\subsection{{Motivation: Toward Optimal and Robust Federated PP}}
 {Prior work in FL or decentralized settings~\citep{jin2023performative,NEURIPS2022_fad7c708} aims for the PS point, which may not minimize the PR. We instead seek the PO point, which minimizes the PR.}

We introduce \textsc{PoFL} (\textbf{P}erformative \textbf{O}ptimal \textbf{F}ederated \textbf{L}earning), a federated PP algorithm that targets the global PO point (see Algorithm~\ref{alg:pofl}). Unlike \textsc{PFL}~\citep{jin2023performative}, which uses only the direct loss-gradient, \textsc{PoFL} estimates the full performative gradient (PG), accounting for model-induced distribution shifts. Estimating the full PG in FL is challenging, as it requires inferring distribution parameters (mean, variance, Jacobian) from local data (see Section~\ref{sec:profl_algorithm}). Errors in these estimates can make PG estimation inaccurate and unstable, especially with nonlinear distribution maps, small sample sizes, or contaminated data.

To address these challenges, we propose \textsc{ProFL} (\textbf{P}erformative \textbf{r}obust \textbf{o}ptimal \textbf{F}ederated \textbf{L}earning), with pseudocode given in
Algorithm~\ref{alg:one}. \textsc{ProFL} incorporates three strategies, detailed in Section~\ref{subsec:robust_gradient_strategy}, to improve the reliability of PG estimation: (i) robust gradient estimation to mitigate the effect of data contamination, (ii) adaptive sample-size selection to balance estimation accuracy and computational cost, and
(iii) server-side aggregation of Jacobian estimates to reduce estimation errors when local samples are limited.

\subsection{Federated Estimation of the Performative Gradient}\label{sec:profl_algorithm}
 The key to finding the PO point is to update the model using the PG, which captures both the direct loss gradient and the effect of model-induced distributional shifts. While \citep{pmlr-v139-izzo21a} proposed a method to estimate the PG, their analysis is limited to a centralized setting with a specific distribution map \( D(\theta) \), where \( f(\theta) \) is assumed to be a linear function of \( \theta \), and the performative risk (PR) is convex. However, such assumptions are restrictive and often fail to hold in practice, especially under data contamination. In FL, the challenge becomes even greater. Unlike  the centralized setting, which computes a single PG \( \nabla \mathcal{L}(\theta) \) on the server, FL requires each client to update its model using a local PG: $\nabla \mathcal{L}_i(\theta)$. Assume the local performative distribution $D_i(\theta)$ has density $p(z;f_i(\theta))$ with parameter $f_i(\theta)$ being a function of $\theta$, then $\mathcal{L}_i(\theta) = \int \ell (z;\theta)p(z;f_i(\theta))dz $ when functions $p$ and $f_i$ are continuously differentiable PG can be derived
 \begin{eqnarray}\label{eq:PG}
     \nabla \mathcal{L}_i(\theta) = \nabla \mathcal{L}_{i,1}(\theta) + \nabla \mathcal{L}_{i,2}(\theta)
 \end{eqnarray}
 where
 \begin{equation*}
     \nabla \mathcal{L}_{i,1}(\theta) = \mathbb{E}_{Z_i \sim D_i(\mathbf{\theta})}[\nabla \ell (Z_i; \mathbf{\theta})]
 \end{equation*}
\begin{equation*}\fitwidth{\hsize}{!}{ $\displaystyle\nabla \mathcal{L}_{i,2}(\theta) = \mathbb{E}_{Z_i \sim D_i(\theta)} \left[ \ell(Z_i;\theta) \left(\frac{df_i(\theta)}{d\theta}\right)^T \frac{\partial \log p\left(Z_i;f_i(\theta)\right)}{\partial f_i(\theta)}
 \right]$}\notag
 \end{equation*}
Note that \textsc{PFL}~\citep{jin2023performative} only updates the model using \( \nabla \mathcal{L}_{i,1}(\theta) \) and neglects \( \nabla \mathcal{L}_{i,2}(\theta) \). {As a result the algorithm may converge to a PS point that is far from PO point. }
In {both \textsc{PoFL} and \textsc{ProFL}}, each active client $i$ in set $\mathcal{I}_t\subseteq \mathcal{V}$, after synchronization with the current global model $\overline{\theta}^t$, uses its local data to update the model from $\theta_i^t$ to $\theta_i^{t+R}$. During these updates, the data distribution evolves with the model, and consequently, client $i$ uses a sample set $\mathcal{S}_i$ drawn from the contaminated distribution $P_i(\theta_i^{t})$ to update its local model $\theta_i^{t}$. To find the optimal ${\theta}^{\text{PO}} = \arg\min_{\theta}\sum_{i\in\mathcal{V}}\alpha_i\mathcal{L}_i(\theta)$, each client $i$ should estimate its local PG $\nabla\mathcal{L}_i(\theta_i^{t})$   to update the local model $\theta_i^{t}$,
\begin{equation}\label{eq:update}
   \theta^{t+1}_i \leftarrow  \theta^{t}_i - \eta\nabla \mathcal{L}_i(\theta^{t}_i).
\end{equation}
By Eq.~\eqref{eq:PG}, $\nabla\mathcal{L}_i(\theta_{i}^{t})=\nabla\mathcal{L}_{i,1}(\theta_{i}^{t})+\nabla\mathcal{L}_{i,2}(\theta_{i}^{t}).$

\textsc{PoFL} and \textsc{ProFL} share the same PG-estimation and model-update framework, but differ in the sample set used for estimating PG. In \textsc{PoFL}, the empirical quantities are computed directly from the \textit{original} sample set \(\mathcal{S}_i\). In \textsc{ProFL}, the robust-gradient procedure first filters \(\mathcal{S}_i\) and returns a \textit{filtered} subset \[ \mathcal{S}'_i \leftarrow \textsc{Robust Gradient} \left( \mathcal{S}_i; \theta_i^t; J_1; J_2; B;q \right). \] Here, \(J_1\) and \(J_2\) are hyperparameters controlling the selection
criteria, \(B\) is the number of histogram bins used to select the score threshold, and \(q\) is the number of power iterations used to approximate the leading singular direction. The robust-gradient procedure is described in detail Section~\ref{subsec:robust_gradient_strategy}. For a unified presentation, define the effective sample set \begin{equation}
\label{eq:effective_sample_set}
\widetilde{\mathcal{S}}_i :=
\begin{cases}
\mathcal{S}_i, & \text{for \textsc{PoFL}},\\
\mathcal{S}'_i, & \text{for \textsc{ProFL}}.
\end{cases}
\end{equation}

For reference, we denote the empirical gradient estimate on samples drawn from distributions \( Q_i \) and \( D_i(\theta_i^t) \) as \( \widehat{\nabla} \mathcal{L}_{i,1}(Q_i, \theta_i^t) \) and \( \widehat{\nabla} \mathcal{L}_{i,1}(D_i(\theta_i^t), \theta_i^t) \), respectively. Similarly, the empirical risk over local samples is denoted as \( \widehat{\mathcal{L}}_{i,1}(\theta_i^t) \), computed by averaging \( \ell(z_j; \theta_i^t) \) over {\( z_j \in \widetilde{\mathcal{S}}_i \).}

We also define \( \widehat{\nabla} \mathcal{L}_{i,2}(\theta_i^t) \) and its counterparts \( \widehat{\nabla} \mathcal{L}_{i,2}(Q_i, \theta_i^t) \) and \( \widehat{\nabla} \mathcal{L}_{i,2}(D_i(\theta_i^t), \theta_i^t) \) analogously. Specifically, \( \widehat{\nabla} \mathcal{L}_{i,2}(\theta_i^t) \) is computed as a \textit{weighted} average of \( \ell(z_j; \theta_i^t) \) over {the sample set \( \widetilde{\mathcal{S}}_i \)}, with weights determined by sensitivity to model-induced distributional parameters. The weight vector is defined as \begin{equation}\fitwidth{0.9\hsize}{!}{
$\displaystyle w(z_j;\theta_i^{t}):=\left.\left(\frac{df_i(\theta)}{d\theta}\right)^T\frac{\partial \log p(z_j;f_i(\theta))}{\partial f_i(\theta)}\right\vert_{\theta = \mathbf{\theta}_i^{t}}\in\mathbb{R}^d$}\label{eq:w},
\end{equation} which encodes how the density \( p(z_j; f_i(\theta)) \) changes with respect to the model parameter \( \theta \) through the mapping \( f_i(\theta) \). Since the performative distribution \( D_i(\theta) = p(z; f_i(\theta)) \) and \( f_i(\theta) \) are both unknown in practice, we must estimate both \( f_i(\theta) \) and its Jacobian \( \frac{d f_i(\theta)}{d \theta} \) from the observed data.
Following \citep{pmlr-v139-izzo21a}, we adopt finite difference approximation $\left.\frac{df_i(\theta)}{d\theta}\right|_{\theta=\theta^{t}_i}\approx \Delta f_i (\Delta \theta_i)^\dagger$. Specifically, at iteration $t$, each client $i$ collects the $H$ most recent models. The models exclude $\theta^{t}_i$ and $H \geq d$. Then forms a matrix $\begin{bmatrix}
     \theta^{t-H}_i & \cdots &  \theta^{t-1}_i
    \end{bmatrix}$, then
\begin{align}\label{eq:delta}
\Delta \theta_i &= \begin{bmatrix}
     \theta^{t-H}_i & \cdots &  \theta^{t-1}_i
    \end{bmatrix} - \theta^{t}_i\textbf{1}^T_H;\\
    \Delta f_i&=  \begin{bmatrix}
      f_i(\theta^{t-H}_i)  &\cdots  & f_i(\theta^{t-1}_i)
    \end{bmatrix} -f_i(\theta^{t}_i)\textbf{1}^T_H, \nonumber
\end{align}
Where $\textbf{1}_H$ is an $H$-dimensional all-ones vector, both $\Delta f_i$ and $\Delta \theta_i$ are $d \times H$ matrices, and $(\Delta \theta_i)^\dagger$ is the pseudo-inverse of $\Delta \theta_i$. For each iteration of each client, we know $\theta_i^t$, but do not know the exact value of $f_i(\theta_i^t)$. Denote the estimator as {$\widehat{f}_i(\widetilde{\mathcal{S}}_i)$}, meaning we use the dataset {$\widetilde{\mathcal{S}}_i$} and the function $\widehat{f}_i(.)$ to estimate the value of $f_i(\theta_i^t)$. This estimator is then used to compute an estimate of $\frac{df_i}{d\theta}$. The estimator of $\frac{df_i}{d\theta}$ is denoted as $\frac{\widehat{df_i}}{d\theta}$ and can be computed as $\frac{\widehat{df_i}}{d\theta}
\approx\Delta\widehat{f}_i(\Delta\theta_i)^\dagger$. Finally, $\frac{\widehat{df_i}}{d\theta}$ will be used to calculate $\widehat{\nabla}\mathcal{L}_{i,2}$. By combining this with $\widehat{\nabla}\mathcal{L}_{i,1}$, we obtain the PG, allowing each client to update the local model $\theta_i^t$ using Eq.~\eqref{eq:update}.

\begin{algorithm}[!htbp]
{
\footnotesize
\caption{{\textsc{PoFL}}}
\label{alg:pofl}
\begin{algorithmic}[1]

\STATE{\textbf{Input:} Learning rate $\eta$, estimation window $H$,
function $\widehat{f}$, total iterations $T$, local update iterations
$R$, and local sample size $n_i$ for all $i\in\mathcal{V}$.}

\FOR{$t = 0$ to $T-1$}

    \STATE If $t \bmod R = 0$, the server sends
    $\overline{\theta}^t$ to a selected set
    $\mathcal{I}_t \subseteq \mathcal{V}$ of clients, and each client
    $i \in \mathcal{I}_t$ sets
    $\theta_i^t = \overline{\theta}^t$; otherwise,
    $\mathcal{I}_t = \mathcal{I}_{t-1}$.

    \FOR{$i \in \mathcal{I}_t$ in parallel}

        \STATE Deploy $\theta_i^t$, allow the local environment to
        respond, and draw $n_i$ samples
        $\mathcal{S}_i := (z_j)_{j=1}^{n_i}
        \mathop{\sim}\limits^{\mathrm{i.i.d.}}
        P_i(\theta_i^t)$.

        \IF{$t < H$}

            \STATE Compute the estimator
            $\widehat{f}(\mathcal{S}_i)$ of $f_i^t$ and the gradient
            estimate
            $\widehat{\nabla}\mathcal{L}_{i,1}(\theta_i^t)$
            using $\mathcal{S}_i$.

            \STATE Update
            $\theta_i^{t+1}
            \leftarrow
            \theta_i^t
            -
            \eta
            \widehat{\nabla}\mathcal{L}_{i,1}(\theta_i^t)
           $.

        \ELSE

            \STATE Compute the estimator
            $\widehat{f}(\mathcal{S}_i)$ of $f_i^t$ and the gradient
            estimate
            $\widehat{\nabla}\mathcal{L}_{i,1}(\theta_i^t)$
            using $\mathcal{S}_i$.

            \STATE Compute
            $\Delta\widehat{f}_i(\Delta\theta_i)^\dagger$
            using $\Delta\theta_i$ and $\Delta f_i$.

            \STATE Compute
            $\widehat{\nabla}\mathcal{L}_{i,2}(\theta_i^t)$
            using the estimated Jacobian.

            \STATE Update
            $\theta_i^{t+1}
            \leftarrow
            \theta_i^t
            -
            \eta
            \left(
            \widehat{\nabla}\mathcal{L}_{i,1}(\theta_i^t)
            +
            \widehat{\nabla}\mathcal{L}_{i,2}(\theta_i^t)
            \right)$.

        \ENDIF

    \ENDFOR

    \STATE If $(t+1)\bmod R = 0$, each client
    $i \in \mathcal{I}_t$ sends $\theta_i^{t+1}$ to the server, and
    the server updates
    $\overline{\theta}^{t+1}
    =
    \sum_{i\in\mathcal{I}_t}
    \alpha_i\theta_i^{t+1}$.

\ENDFOR

\STATE{\textbf{Output:} $\overline{\theta}^{T}$.}

\end{algorithmic}
}
\end{algorithm}

\begin{algorithm}[!htbp]
{\footnotesize
\caption{\textsc{ProFL}}
\label{alg:one}
\begin{algorithmic}[1]
\STATE{\textbf{Input:} Initial model $\overline{\theta}^0$, learning rate $\eta$, estimation window $H$,
function $\widehat{f}$, total iterations $T$, local update iterations
$R$, error bound $\Phi_i$ for all $i \in \mathcal{V}$, thresholds
$J_1,J_2 \in (0,1)$, {number of histogram bins $B$,} number of
clusters $C$, {number of power iterations $q$.}}
\FOR{$t = 0$ to $T-1$}
    \STATE If $t \mod R = 0$, server sends $\overline{\theta}^t$ to a selected set $\mathcal{I}_t \subseteq \mathcal{V}$ of clients and client $i \in \mathcal{I}_t$ sets the local model $\theta^{t}_i = \overline{\theta}^t$; otherwise $\mathcal{I}_t = \mathcal{I}_{t-1}.$
    \FOR {$i \in \mathcal{I}_t$ in parallel}
        \STATE Deploy $\theta^{t}_i$ and local environment changes and draw $n_i$ samples $\mathcal{S}_i:=(z_j)^{n_i}_{j=1} \mathop{\sim}\limits^{iid} P_i(\theta^{t}_i)$.
        \IF{$t < H$}

                 \STATE $\widehat{\nabla}\mathcal{L}_{i,1}(\theta^{t}_i); \mathcal{S}'_i \leftarrow \textsc{Robust Gradient}\left(\mathcal{S}_i;\theta^{t}_i;J_1;J_2;{B;q} \right).$
                 \STATE Compute estimator $\widehat{f}\left(\mathcal{S'}_i\right)$ of $f^{t}_i$.
                 \STATE Update $\theta^{t+1}_i \leftarrow \theta^{t}_i - \eta\widehat{\nabla} \mathcal{L}_{i,1}(\theta^{t}_i)$
                 \ELSE
                 \STATE \fitwidth{0.825\hsize}{!}{ $\widehat{\mathcal{L}}_{i,1}(\theta^{t}_i);\widehat{\nabla} \mathcal{L}_{i,1}(\theta^{t}_i); \mathcal{S}'_i \leftarrow \textsc{Robust Gradient}
                \left(\mathcal{S}_i;\theta_i^t;J_1;J_2;B;q\right).$}
            \STATE Compute estimator $\widehat{f}\left(\mathcal{S'}_i\right)$ of $f^{t}_i$.
            \STATE Compute the local Jacobian estimate $\widehat{\frac{df_i}{d\theta}} \leftarrow \Delta\widehat{f}_i(\Delta\theta_i)^\dagger$.
            {\IF{local sample size is small}
                \STATE Send $\widehat{\frac{df_i}{d\theta}}$ to the server.
                \STATE $\widehat{\frac{df_c}{d\theta}} \leftarrow$
                \textsc{ServerAggregation}$(\{\widehat{\frac{df_i}{d\theta}}\}_{i\in\mathcal I_t}, C)$.
                \STATE Compute $\widehat{\nabla}\mathcal{L}_{i,2}(\theta_i^t)$
                using the aggregated $\widehat{\frac{df_c}{d\theta}}$.
            \ELSE
                \STATE Compute $\widehat{\nabla}\mathcal{L}_{i,2}(\theta_i^t)$
                using the local estimate $\widehat{\frac{df_i}{d\theta}}$.
            \ENDIF
            \STATE $\theta^{t+1}_i \leftarrow \theta^{t}_i - \eta\left(\widehat{\nabla} \mathcal{L}_{i,1}(\theta^{t}_i) + \widehat{\nabla}\mathcal{L}_{i,2}(\theta^{t}_i)\right)$.
        \STATE \texttt{Input}$ =\{\widehat{\mathcal{L}}_{i,1}(\theta^{t}_i); H;d;\eta; \varphi_i; \Delta\widehat{f}_i(\Delta\theta_i)^\dagger; \Phi_i\}$
                \STATE {$n_i \leftarrow \textsc{Adaptive Sample Size}\left(\texttt{Input}\right)$.}  }
        \ENDIF
    \ENDFOR
    \STATE If $(t+1) \mod R = 0$, client $i \in \mathcal{I}_t$ sends $\theta^{t+1}_i$ to the server and server updates the model $\overline{\theta}^{t+1} = \sum_{i \in \mathcal{I}_t} \alpha_i \theta^{t+1}_i$.
\ENDFOR
\STATE{\textbf{Output:}} $\overline{\theta}^{T}$
{\STATE \textbf{Server Aggregation} \textit{(details of Line~16).}}
\STATE Use \( k \)-means clustering to group clients into \( C \) clusters.
\STATE Compute the aggregated gradient estimate \( \widehat{\frac{df_c}{d\theta}} \) by averaging \( \widehat{\frac{df_i}{d\theta}} \) for all \( i \) in cluster \( c \), for each cluster \( c \in \{1, \ldots, C\} \).
\STATE Send \( \widehat{\frac{df_c}{d\theta}} \) back to the clients belonging to cluster \( c \) for all \( c \in \{1, \ldots, C\} \).
\end{algorithmic}}
\end{algorithm}

\subsection{{Walkthrough and Additional Strategies \textsc{ProFL}}}\label{subsec:robust_gradient_strategy}
While \citep{pmlr-v139-izzo21a} showed their method for estimating \( f_i(\theta) \) and its Jacobian \( \frac{d f_i(\theta)}{d \theta} \) works in centralized settings, their analysis assumed linearity and clean data. In FL, distribution shifts are often nonlinear and local data can be limited or noisy, which makes their PG estimation less accurate. To address this, \textsc{ProFL} uses several robustness strategies described below.

\textbf{Walkthrough.}
We briefly summarize one outer iteration of Algorithm~\ref{alg:one}. Every $R$ local steps, the server sends the global model $\overline{\theta}^t$ to a subset of clients. Each client sets $\theta_i^{t}=\overline{\theta}^t$, collects samples from $P_i(\theta_i^t)$, and computes the direct-gradient estimate
$\widehat{\nabla}\mathcal{L}_{i,1}(\theta_i^t)$ using the \textit{robust-gradient} procedure. For the first $H$ iterations, only the direct gradient is used. Once sufficient history has been collected, the client estimates $\frac{df_i}{d\theta}$ using recent data. If the sample size is too small, the client sends its estimate to the server for aggregation; otherwise, it uses its own. This estimate is used to compute
$\widehat{\nabla}\mathcal{L}_{i,2}(\theta_i^t)$. The client then updates using both gradients and can adjust its sample size for the next round. Every $R$ steps, clients send their updated models to the server, which aggregates them into an updated global model.

\textbf{Non-linearity of  $f_i(\theta)$.} {When $f_i(\theta)$ is non-linear, the first-order finite-difference estimate of $\widehat{\frac{df_i}{d\theta}}$ may incur a bias due to the higher-order Taylor remainder.} While \citep{pmlr-v139-izzo21a} demonstrated that this method works well for linear $f_i(\theta)$, the impact of non-linearity on performance remains unclear. To mitigate the error introduced by non-linearity, \textsc{ProFL} employs the following strategies: (i) reducing the learning rate $\eta$ and the estimation window $H$; (ii) estimating $\widehat{\frac{df_i}{d\theta}}$ at the server. Reducing $\eta$ and $H$ helps reduce the estimation error resulting from non-linear $f_i(\theta)$ and stabilizes the algorithm.

When the local sample size is small, $\widehat{f}_i(\mathcal{S'}_i)$ may significantly deviate from $f_i(\theta)$, leading to poor estimation of $\frac{df_i}{d\theta}$. Estimating $\widehat{\frac{df_i}{d\theta}}$ at the server mitigates the negative impact of sampling error.
Therefore, \textsc{ProFL} uses the \textsc{ServerAggregation} subroutine in
Algorithm~\ref{alg:one}. Clients send their local estimates $\widehat{\frac{df_i}{d\theta}}$ to the server,which aggregates these estimates within clusters of clients with similar distribution dynamics (Lines~15--16). The detailed aggregation procedure is given in Lines~29--32. Specifically, the server uses $k$-means clustering to divide the selected clients into $C$ clusters and averages the estimates $\widehat{\frac{df_i}{d\theta}}$ within each cluster. The cluster-level estimate $\widehat{\frac{df_c}{d\theta}}$ is then send back to clients to compute $\widehat{\nabla}\mathcal{L}_{i,2}(\theta_i^t)$. $C$ is the number of clusters and controls the granularity of server-side aggregation. A smaller $C$ averages more clients together and reduces estimation variance, while a larger $C$ preserves more client heterogeneity. When $C=1$, all selected clients are aggregated together, and when $C=|\mathcal I_t|$, no cross-client aggregation is performed.

\noindent\textit{Compatibility with secure aggregation.}
To protect the shared Jacobian estimates, the server broadcasts the
previous cluster centers, and each client locally selects its nearest
cluster without revealing its individual Jacobian. The client then
provides its cluster identity and masked Jacobian to a cluster-wise secure
aggregation protocol. For each cluster, the server obtains only the
aggregated Jacobian sum and client count, from which it computes
$\widehat{\frac{df_c}{d\theta}}$ and updates the cluster center for the
next round. Thus, individual Jacobian estimates remain hidden, although
the server may still observe cluster membership and cluster-level
aggregates. Undersized clusters should be merged or excluded to satisfy
the minimum-size requirements of secure aggregation
\cite{bonawitz2017practical}.

\textbf{Balancing the impact of sampling error and computational costs.} In practice, \( \frac{df_i(\theta)}{d\theta} \) must be estimated from finite samples, which introduces sampling error into the computation of the PG \( \nabla \mathcal{L}_i(\theta_i^t) \). As shown in Lemmas~\ref{le:error df dfdtheta}, \ref{le:error L1} and ~\ref{le:error L2D with h} in Section \ref{sec:theorem}, the sampling error in estimating $\frac{df_i(\theta)}{d\theta}$ is \textit{amplified} when computing $\nabla\mathcal{L}_{i,2}(\theta_{i}^{t})$, making the PG and algorithm more sensitive to sampling errors. To mitigate this, \textsc{ProFL} adaptively selects the sample size $n_i$ based on the error bounds of the PG and local loss. While a larger $n_i$ generally leads to a more accurate estimate, it may not always be feasible in FL due to clients' limited computational capabilities.
Therefore, \textsc{ProFL} balances estimation accuracy and resource constraints by computing a theoretical lower bound on \( n_i \) using quantities such as model variance, loss curvature, learning rate, and empirical gradient variability. As these values change during updates, $n_i$ should be selected adaptively.
{Different from sample-complexity studies for stochastic policy-gradient
methods~\citep{fatkhullin2023stochastic}, we adaptively select local sample sizes to control performative-gradient estimation error in FL. The \textsc{Adaptive Sample Size} subroutine in Algorithm~\ref{alg:one} implements this strategy by updating $n_i$ according to the target estimation-error tolerance $\Phi_i$. The update uses the quantities appearing in the theoretical estimation-error bound, including $\widehat{\mathcal{L}}_{i,1}(\theta_i^t)$, $H$, $d$, $\eta$, $\varphi_i$, and the local Jacobian estimate
$\Delta \widehat f_i(\Delta\theta_i)^\dagger$. The explicit lower bound for choosing $n_i$ is derived in Section~\ref{subsec:discussion}.}

 \textbf{Handling data contamination.}
  Due to contamination $Q_i$, the data used to estimating $\frac{df(\theta)}{d\theta}$ and $f_i(\theta)$ are drawn from
$P_i(\theta)$ rather than the clean distribution $D_i(\theta)$. Even a small update error from one client can cascade and impact future data and the entire network. In FL, because data must remain local for privacy reasons, contamination must also be handled locally rather than at the central server. Accordingly, \textsc{ProFL} uses
local gradient filtering to mitigate the influence of high-impact
contamination.

To estimate $\nabla\mathcal{L}_{i,1}$, we should average
$\nabla\ell(z_j;\theta_i^t)$ over samples drawn from
$D_i(\theta_i^t)$; however, the samples actually collected by client
$i$ are drawn from $P_i(\theta_i^t)$ and thus include an unknown original contamination fraction
$\epsilon_i^{\mathrm{orig}}$ associated with an unknown distribution
$Q_i$ As specified in the system model, \textsc{ProFL}
does not seek to recover the $D_i/Q_i$ membership of every sample.
Instead, it filters samples whose gradients exhibit the high-impact
directional and tail behavior targeted by the proposed robust-gradient
mechanism.

The robust-gradient component of \textsc{ProFL} is
modular and can be replaced by other local robust estimators. In our paper, we assume that contaminated data can follow an \textit{arbitrary} distribution. In practice, additional information about the contamination process may be available based on the application or context, allowing us to design more efficient and less expensive outlier identification mechanisms leveraging this knowledge.

\begin{algorithm}[!htbp]
\small
\caption{\textsc{Robust Gradient} {via SVD-Score Filtering}}\label{alg:two}
\begin{algorithmic}[1]
\STATE{\textbf{Input:} Set $\mathcal{S}_i$ of samples $(z_j)_{j=1}^{n_i}$, thresholds $J_1, J_2 \in (0,1)$, {number of histogram bins $B$, number of power iterations $q$, }model $\theta^{t}_i \in \Theta$}
\STATE Let $\widehat{\nabla} \leftarrow \frac{1}{|\mathcal{S}_i|} \sum_{z_j\in \mathcal{S}_i} \nabla \ell(z_j;\theta^{t}_i)$.
\STATE Let $ \left[\nabla \ell(z_j;\theta^{t}_i) - \widehat{\nabla}\right]_{z_j\in \mathcal{S}_i}$ be the $|\mathcal{S}_i| \times d$ matrix {$G_i$}.
\STATE {Compute or approximate the leading right singular vector $v$ of $G_i$; in our implementation, $v$ is approximated using $q$ power iterations.}
\STATE Compute $\tau_j = \left( \left(\nabla \ell(z_j;\theta^{t}_i) - \widehat{\nabla}\right)^T v \right)^2$ of $\nabla \ell(z_j;\theta^{t}_i)$.
\STATE Divide the interval $[0, \max  \tau_j]$ into $B$ equal-length segments. Let $\phi_k $ be the number of $\tau_j$ in the $k$-th segment.
\STATE Find the smallest $k\in \{1,\cdots,B\}$ with $\phi_k < J_1 \cdot |\mathcal{S}_i|$. Set $\phi$ as the lower bound of the $k$-th segment. If no such $k$ exists, return the current
$\mathcal{S}_i$ and $\widehat{\nabla}$.
\STATE Find the set $\mathcal{S}_i' \leftarrow \{z_j \in \mathcal{S}_i: \tau_j < \phi\}$.
\STATE Compute $\widehat{\nabla}' \leftarrow
\frac{1}{|\mathcal{S}'_i|}
\sum_{z_j\in\mathcal{S}'_i}\nabla\ell(z_j;\theta_i^t)$.
\STATE Compare the average gradients $\widehat{\nabla}$ of $\mathcal{S}_i$ and $\widehat{\nabla}'$ of $\mathcal{S}'_i$. If $\frac{\|\widehat{\nabla}-\widehat{\nabla}'\|}{\|\widehat{\nabla}\|}<J_2$ or $|\mathcal{S}'_i|\leq\frac{n_i}{2}$, return $\mathcal{S}'_i$ and $\widehat{\nabla}'$; otherwise, set $\mathcal{S}_i\leftarrow\mathcal{S}'_i$ and repeat from line 2.
\end{algorithmic}
\normalsize
\end{algorithm}

As shown in Algorithm~\ref{alg:two}, our mechanism takes noisy samples
$\mathcal{S}_i$ and current local model $\theta_i^{t}$ as inputs. The
mechanism first computes the gradients
$\nabla \ell(z_j;\theta_i^{t})$ of all samples and iteratively
filters gradients with large SVD scores. Since both the distribution $Q_i$ and original contamination fraction
$\epsilon_i^{\mathrm{orig}}$ are unknown, we target high-impact gradients that exhibit
two crucial properties: 1) they have large effects on the learned model
and can disturb the training process (i.e., gradients have large
magnitudes and systematically point in a specific direction); 2) they
are located farther away from the average in the vector space (i.e.,
``long tail'' data).
To detect gradients that satisfy these properties, we adapt the
approach in~\citep{pmlr-v97-diakonikolas19a}, which uses the leading
singular direction of the centered gradient matrix.

Specifically, we construct a matrix ${G_i=}[ \nabla \ell(z_j;\theta^{t}_i) - \widehat{\nabla}]_{z_j\in \mathcal{S}_i}$ {and compute its leading right singular vector $v$, where
$\widehat{\nabla}$ is the average gradient over the current sample set
(lines 2--3). Since only the leading singular direction is required,
a full SVD is unnecessary. In our implementation, we approximate $v$
using $q$ power iterations.}
with larger absolute projections onto $v$ receive larger
scores and are treated as more likely high-impact candidates, and we assign each gradient an outlier score $\tau_j$ as follows:
\begin{align}\label{eq:scores}
\tau_j = {\left(  \left(\nabla \ell(z_j;\theta^{t}_i) - \widehat{\nabla}\right)^Tv \right)}^2
\end{align}
Given the outlier scores $\{\tau_j\}$, the gradients with scores above a threshold $\phi$ are treated as high-impact candidates and discarded.

\noindent\textbf{Key Innovations of Robust Gradient in \textsc{ProFL}.}
Unlike the method in~\citep{pmlr-v97-diakonikolas19a}, which requires a manually preselected threshold and resampling after each removal, \textsc{ProFL} determines the outlier threshold \( \phi \) \textit{automatically} from the empirical distribution of outlier scores \( \tau_j \) and uses the same dataset to get the {robust gradient} because it is difficult to determine without any prior knowledge of the data and resampling new data after each removal step is impractical in PP.

Since we assume the outliers lie in the ``long tail,'' we divide the range \( [0, \max \tau_j] \) into \( B \) equal-length segments and compute the relative frequency of each. The threshold \( \phi \) is chosen as the smallest segment whose relative frequency falls below a predefined level. This method exploits the \textit{score concentration} and \textit{tail sparsity} properties, making the filtering process predictable and stable. The final robust gradient estimate \( \widehat{\nabla}' \) is computed by averaging gradients with scores below \( \phi \), mitigating the influence of high-score gradients without
requiring prior knowledge of the contamination fraction.

{A key challenge in robust-gradient estimation is that both the contamination fraction and the scale of abnormal gradients are typically unknown. This is especially relevant in FL, where clients observe only local samples and usually cannot depend on centralized preprocessing or manually cleaned reference data. Thus, methods requiring a prescribed contamination ratio or absolute filtering threshold are difficult to deploy when local gradient distributions vary across clients and training rounds \citep{shan2023beyond}. Therefore, we propose Algorithm~\ref{alg:two}, which does not require prior knowledge of the contamination fraction or a fixed absolute threshold on the spectral scores. Instead, the filtering threshold $\phi$ is chosen adaptively from the empirical SVD score distribution at each client and round. The parameter $J_1$ identifies a low-frequency region in the score histogram, while $J_2$ terminates filtering when the resulting gradient estimate becomes sufficiently stable. Consequently, both the number of removed samples and the number of filtering rounds adapt to the observed local gradient distribution.}

Importantly, the robust-gradient procedure only requires the \textit{leading singular direction} rather than the full SVD. In Algorithm~\ref{alg:two}, this direction is approximated using a fixed number $q$ of power iterations. The parameter $q$ controls the tradeoff between the accuracy of the leading-direction approximation and the local computational cost. In practice, $q$ can be chosen much smaller than the local sample size $n_i$.

\section{Performance Analysis}\label{sec:theorem}

To facilitate analysis, we study two special cases {that provide explicit
instantiations of the distribution-dependent quantities in the performative
gradient}: (i) \textit{Contribution dynamics}, where only
$\nu_{i,k}(\theta)$ changes while the group distribution remains fixed, i.e.,
$D_i(\theta)=\sum_{k\in[K]}\nu_{i,k}(\theta)D_{i,k}$. In this case,
$f_{i,k}(\theta) = \nu_{i,k}(\theta)$, and $\widehat{f}_{i,k}$ estimates the
sample proportion from group $k$. (ii) \textit{Distribution dynamics}, where
only the distribution $D_{i,k}(\theta)$ changes while the contribution from
each group remains fixed, i.e.,
$D_i(\theta)=\sum_{k\in[K]}\nu_{i,k}D_{i,k}(\theta)$. We consider a mixture of
Gaussians with $D_{i,k}(\theta) = \mathcal{N}(f_k(\theta), \sigma_k^2)$, where
$f_{i,k}(\theta)$ is the mean of group $k$ for client $i$, and the overall
mean for client $i$ is
$f_i(\theta) = \sum_{k \in [K]} \nu_{i,k} f_{i,k}(\theta)$.
$\widehat{f}_i(\widetilde{\mathcal{S}}_i)$ represents the empirical mean from {effective sample set $\widetilde{\mathcal{S}}_i$. Estimator-error bounds are derived explicitly for both cases. More generally, the convergence analysis below uses these estimator-error bounds as inputs and is not restricted to these two distribution dynamics. Other performative distribution models can be handled using the same analytical framework whenever the corresponding smoothness, boundedness, and finite-sample estimator-error conditions can be verified. The constants in the bounds may, however, depend on the chosen distribution family and parameterization $p(z;f_i(\theta))$.}
Before presenting them, we introduce the technical assumptions. For simplicity,
$\| . \|_2$ is denoted as $\| . \|$. Proofs are in Appendix.

\begin{assumption}[]\label{as:Sensitivity} Let $W_1(D,D')$ measure the Wasserstein-1 distance between two distributions $D$ and $D'$. Then $\forall i \in \mathcal{V}$, there exists $\gamma_i > 0$ such that $\forall \theta, \theta' \in \Theta: W_1\left(D_i(\theta),D_i(\theta')\right) \leq \gamma_i {\|\theta-\theta'\|}.$
\end{assumption}
\begin{assumption}[Smoothness]\label{as:smoothness}
There exists $L>0$ such that $\ell(z;\theta)$ is continuously
differentiable and, for all $\theta,\theta'\in\Theta$ and
$z,z'\in\mathcal{Z}, \|\nabla_{\theta}\ell(z;\theta)
-\nabla_{\theta}\ell(z';\theta')\|
\leq
L\bigl(\|\theta-\theta'\|+\|z-z'\|\bigr).$
Moreover, for every $i\in\mathcal{V}$, the local performative
risk $\mathcal{L}_i$ is $L(1+\gamma_i)$-smooth, i.e.,$\|\nabla\mathcal{L}_i(\theta)
-\nabla\mathcal{L}_i(\theta')\|
\leq
L(1+\gamma_i)\|\theta-\theta'\|,
\forall\theta,\theta'\in\Theta.$
\end{assumption}

\begin{assumption}[]\label{as:Twice continuously differentiable of f}
$f_i(\theta)$ is twice continuously differentiable for all $\theta \in \Theta,$ i.e., the first and second derivatives $\frac{d f_i(\theta)}{d \theta}$ and $\frac{d^2 f_i(\theta)}{d \theta^2}$ exist and are continuous.
\end{assumption}

\begin{assumption}[PL Condition \citep{polyak1964some}]
\label{as:PL}
The {global performative risk $\mathcal{L}(\theta)$} satisfies the
Polyak--{\L}ojasiewicz (PL) condition, i.e., there exists $\rho>0$ such
that for all $\theta \in \Theta$,
\begin{align*}
   { \frac{1}{2}
    \left\|
        \nabla \mathcal{L}(\theta)
    \right\|^2
    \geq
    \rho
    \left(
        \mathcal{L}(\theta)
        -
        \mathcal{L}(\theta^{\mathrm{PO}})
    \right).}
\end{align*}

\end{assumption}
{Unlike existing works that impose convexity assumptions on local performative risks, Assumption~\ref{as:PL} only requires the global
performative risk to satisfy the PL condition, which can hold for \textit{non-convex} objectives.}

\subsection{Error of the performative gradient}\label{subsec:lemmas}
As discussed in Section~\ref{sec:algo}, the key to
optimizing toward $\theta^{\text{PO}}$ is estimating the
PG
$\nabla\mathcal{L}_i(\theta_i^{t})
=
\nabla\mathcal{L}_{i,1}(\theta_i^{t})
+
\nabla\mathcal{L}_{i,2}(\theta_i^{t})$
in Eq.~\eqref{eq:PG} accurately. Thus, we analyze the estimation errors
of $\nabla\mathcal{L}_{i,1}(\theta_i^{t})$ and
$\nabla\mathcal{L}_{i,2}(\theta_i^{t})$ and explore how these errors are
influenced by the non-linearity of $f_i(\theta)$, sampling error, and
contaminated data $Q_i$. The results of this analysis will then be used
to assess the convergence of \textsc{ProFL} (and \textsc{PoFL}) in
Section~\ref{subsec:convergence}.

\begin{lemma}Under Assumption~\ref{as:Twice continuously differentiable of f}\label{le:upper bounds}, there exist $F,M<\infty $ such that $\left\|\frac{df_i(\theta)}{d\theta}\right\|\leq F$ and $\left\|\frac{d^2f_i(\theta)}{d\theta^2}\right\| \leq M$ hold for all $i \in \mathcal{V}.$ Under Assumption~\ref{as:smoothness}, there exist $G,\ell_{\max}<\infty$ such that $\|\nabla \ell (z;\mathbf{\theta})\|\leq G$ and $\ell (z;\mathbf{\theta})\leq \ell_{\max}$ hold for all $z \in \mathcal{Z}$.
\end{lemma}
Lemma~\ref{le:upper bounds} is proved using Weierstrass's theorem. We will use $F,M,G,\ell_{\max}$ to denote the upper bounds of these quantities in the rest of the paper.

{Next, we provide \textit{estimator-error bounds} for the PG in the two representative distribution dynamics. The proofs use two distribution-specific ingredients. First, in Lemma~\ref{le:error df dfdtheta}, the estimator of $f_i(\theta)$ is either an empirical group proportion in contribution dynamics or an empirical mean in Gaussian-mean distribution dynamics. Both are linear statistics of the local samples, which allow the estimator under the
contaminated mixture $P_i(\theta)=(1-\epsilon_i^{\mathrm{orig}})D_i(\theta)+\epsilon_i^{\mathrm{orig}} Q_i$ to be decomposed
into a clean-data estimation term and a contamination-induced term. Second, in Lemma~\ref{le:error L2D with h}, the bound uses an upper bound on the score-function term
$\left\|
\frac{\partial \log p(z;f_i(\theta))}{\partial f_i(\theta)}
\right\|^2$.
For the Gaussian-mean dynamics, this bound is distribution-dependent
and is related to the covariance of the Gaussian distribution.}

{For other distribution families, the same approach applies with the appropriate estimator decomposition, concentration bound, and score-function sensitivity bound. The proof flow is unchanged, though the constants and decomposition details may vary with the distribution family and parameterization $p(z;f_i(\theta))$. Proofs of the following lemmas and theorem are in the Appendix.}

\begin{lemma}[Estimation error of $\frac{\widehat{df_i}}{d\theta}$]\label{le:error df dfdtheta}Under Assumption~\ref{as:Twice continuously differentiable of f}, let $\lambda_{i,\min}>0$ be a deterministic uniform lower
bound satisfying
$\sigma_d(\Delta\theta_i/\eta)\geq\lambda_{i,\min}$
almost surely at every Jacobian-estimation iteration.
With probability at least $1-\varphi_i$, the estimation error of $\frac{\widehat{df_i}}{d\theta}$ is bounded by $\pmb{\omega_F}$ specified as follows:
{\small
\begin{align*}
&\mathbb{E}\left[
\left\|
\frac{\widehat{df}}{d\theta}-\frac{df}{d\theta}
\right\|_F^2\right]\leq
4\epsilon_iF^2+
\frac{32HW_1(D,Q)_{\max}^2}{\eta^2}
\epsilon_i\lambda_{i,\min}^{-2}\\
&\quad+
\left(
M^2\eta^2G^4H^6+
\frac{32Hf_\varphi(\varphi_i;n_i)/n_i+
32H(\omega_{\mathrm{sel},i})^2}{\eta^2}
\right)\lambda_{i,\min}^{-2}.
\end{align*}
}
\end{lemma}
Notably, $f_{\varphi}(\varphi_i;n_i)$ is a function that depends on the properties of the distribution. For instance, $f_{\varphi}(\varphi_i;n_i) = \mathcal{O}(-\log \varphi_i)$ in distribution dynamics, and $f_{\varphi}(\varphi_i;n_i) = \mathcal{O}(1/\varphi_i)$ in contribution dynamics.

\begin{lemma}[Estimation error of $ \nabla \mathcal{L}_{i,1}$]\label{le:error L1} Under Assumption~\ref{as:Twice continuously differentiable of f}, with probability at least $1-\varphi_i$, we have:
\begin{equation}\fitwidth{0.98\hsize}{!}{
    $\displaystyle\mathbb{E}\left[   \left \| \nabla \mathcal{L}_{i,1}(\theta_i^{t})-\widehat{\nabla}\mathcal{L}_{i,1}(D_i(\theta_i^{t}),\theta_i^{t})\right \|^2\right] \leq \mathcal{O}\left( \frac{f_{\varphi}(\varphi_i;n_i)}{n_i} \right).$}\notag
\end{equation}

\end{lemma}

\begin{lemma}[Estimation error of $ \nabla \mathcal{L}_{i,2}$]\label{le:error L2D with h}Under Assumptions~\ref{as:Sensitivity}, \ref{as:smoothness} and \ref{as:Twice continuously differentiable of f}, the following holds with probability at least $1-\varphi_i$:
{\small    \begin{align*}
    &\mathbb{E}\left[   \left \| \nabla \mathcal{L}_{i,2}(\theta_i^{t})-\widehat{\nabla}\mathcal{L}_{{i,2}}(D_i(\theta_i^{t}),\theta_i^{t}) \right \|^2\right]\\ &\leq (\pmb{\omega_D})^2%
    := \mathcal{O}\left(\ell^2_{\max}d \left(\pmb{\omega_F}+\frac{f_{\varphi}(\varphi_i;n_i)F^2}{n_i}\right) \right).\\
    &\mathbb{E}\left[\left \| {\nabla}\mathcal{L}_{{i,2} }(Q_i,\theta_i^{t})-\widehat{\nabla}\mathcal{L}_{{i,2} }(Q_i,\theta_i^{t})\right\|^2 \right]\\ &\leq (\pmb{\omega_Q})^2
    := \mathcal{O}\left(2\ell^2_{\max} d \left(\pmb{\omega_F} + F^2\right) \right).
\end{align*} }
\end{lemma}
Additionally, the notations $\pmb{\omega_D}$ and $\pmb{\omega_Q}$ represent upper bounds. From Lemmas \ref{le:error L1} and
\ref{le:error L2D with h}, we find that since the estimation error of PG is primarily dominated by the estimation error of $ \nabla \mathcal{L}_{i,2}$, we have:
 \begin{align}
    &\mathbb{E}\left[   \left \| \nabla \mathcal{L}_{i,1}(\theta_i^{t})-\widehat{\nabla}\mathcal{L}_{i,1}(D_i(\theta_i^{t}),\theta_i^{t})\right \|^2\right] \\ \notag
    &+ \mathbb{E}\left[   \left \| \nabla \mathcal{L}_{i,2}(\theta_i^{t})-\widehat{\nabla}\mathcal{L}_{{i,2}}(D_i(\theta_i^{t}),\theta_i^{t}) \right \|^2\right] \leq (\pmb{\omega_D})^2.
\end{align}

 Let $\overline{\pmb{\omega_Q}}, \overline{\pmb{\omega_D}}$ be the maximum of $\pmb{\omega_Q}$ and $\pmb{\omega_D}$ respectively in Lemma~\ref{le:error L2D with h} $\forall i\in \mathcal{V}$ and  $\mathcal{S}_{i,Q}^t\subseteq\mathcal{S}_i^t$ denote the
samples drawn from $Q_i$. Define
$\epsilon_i^t=
|\widetilde{\mathcal{S}}_i^t\cap\mathcal{S}_{i,Q}^t|/
|\widetilde{\mathcal{S}}_i^t|$ and
$\overline{\epsilon}^{t}
=\sum_{i\in\mathcal V}\alpha_i\mathbb{E}[\epsilon_i^t]$.

\begin{lemma}\label{le:upper bound L2 with h}
Let $\epsilon_i^t$ be the residual contamination fraction in the
effective sample set at iteration $t$, and let
$\epsilon_i:=\max_{0\leq t<T}\mathbb{E}[\epsilon_i^t]$.
Under Assumption~\ref{as:smoothness}, we have
\begin{align*}
&\mathbb{E}\left[
\left\|
\nabla\mathcal{L}_i(\theta_i^t)-g_i^t
\right\|^2
\,\middle|\,\epsilon_i^t
\right]\\
&\leq
(1-\epsilon_i^t)\overline{\pmb{\omega_D}}^2
+\epsilon_i^t
\left(
LW_1(D,Q)_{\max}
+\overline{\pmb{\omega_D}}
+\overline{\pmb{\omega_Q}}
\right)^2.
\end{align*}
\end{lemma}
$W_1(D,Q)_{\max}$ is the largest Wasserstein-1 distance between
$D_i(\theta_i^t)$ and $Q_i$ for clients $i\in\mathcal{V}$ and
iterations $t\in[0,T-1]$. Since $\overline{\pmb{\omega_D}}^2\leq
\left(LW_1(D,Q)_{\max}+\overline{\pmb{\omega_D}}+
\overline{\pmb{\omega_Q}}\right)^2$, Lemma~\ref{le:upper bound L2 with h}
shows that reducing either $\epsilon_i^t$ or $W_1(D,Q)_{\max}$ tightens
the upper bound on the gradient-estimation error.

\begin{lemma}[One-step descent]
\label{le:one-step-descent}
Under the conditions of Theorem~\ref{th:3}, the expected
objective gap after one update satisfies
\begin{align}
&
\mathbb{E}\left[
\mathcal{L}(\overline{\theta}^{t+1})
-\mathcal{L}(\theta^{\mathrm{PO}})
\right]\notag\\
&\leq
(1-\rho\eta)
\mathbb{E}\left[
\mathcal{L}(\overline{\theta}^{t})
-\mathcal{L}(\theta^{\mathrm{PO}})
\right]\notag\\
&\quad+
\left[
(1-\overline{\epsilon}^{t})\overline{\pmb{\omega_D}}^2
+\overline{\epsilon}^{t}
\left(\overline{\pmb{\omega_{D,Q}}}\right)^2
-\frac{1}{2}\mathbb{E}\left[\|g^t\|^2\right]
\right]\eta\notag\\
&\quad+
L(1+\overline{\gamma})
\left[
(1-\overline{\epsilon}^{t})\overline{\pmb{\omega_D}}^2
+\overline{\epsilon}^{t}
\left(\overline{\pmb{\omega_{D,Q}}}\right)^2
+G^2
\right]\eta^2\notag\\
&\quad+
4L^2(R-1)^2G^2
\left(
(1+\overline{\gamma})^2+\sigma_\gamma^2
\right)\eta^3,
\label{eq:one-step-bound}
\end{align}
where
$\overline{\pmb{\omega_{D,Q}}}
=
\overline{\pmb{\omega_D}}
+
LW_1(D,Q)_{\max}
+
\overline{\pmb{\omega_Q}}$.
\end{lemma}

\subsection{Convergence analysis}\label{subsec:convergence}
We state the
neighborhood-convergence bound in
Theorem~\ref{th:3} for both \textsc{PoFL} and \textsc{ProFL} based on the estimator-error bounds above and explain
how residual contamination affects this bound.
The analysis explicitly retains contamination that is not
distinguished by the filtering procedure through the residual
contamination fraction and the corresponding estimation-error terms.
Finally, we discuss the case without residual contamination
as a clean-data specialization of the general
neighborhood bound.

Denote $\overline{\gamma} = \sum_{i \in \mathcal{V}}\alpha_i\gamma_i$, $\sigma^2_{\gamma} = \sum_{i \in \mathcal{V}}\alpha_i(\overline{\gamma}-\gamma_i)^2$, $\overline{\epsilon}
:=
\max_{0\leq t<T}\overline{\epsilon}^{\,t}$, and $\varphi = \sum_{i \in \mathcal{V}} \varphi_i.$ Throughout the training process, let $W_1(D,Q)_{\max}$ represent the maximum  Wasserstein-1 distance $W_1\left(D(\theta_i^t),Q_i\right)$ between actual data distribution and contamination $\forall i\in \mathcal{V}$. Let $g^2_{\min}$ be the minimal value of $\mathbb{E}[\|g^t\|^2]$ with $g^t = \sum_{i \in \mathcal{V}} \alpha_i(\widehat{\nabla}\mathcal{L}_{i,1}+\widehat{\nabla}\mathcal{L}_{{i,2}}).$

\begin{theorem}[Linear convergence to a neighborhood]
    \label{th:3}
 Let $\{\overline{\theta}^{t}\}_{t\geq 0}$ be a sequence of global models generated by either \textup{\textsc{PoFL}} or \textup{\textsc{ProFL}}.
  Under Assumptions~\ref{as:Sensitivity},
  \ref{as:smoothness},
  \ref{as:Twice continuously differentiable of f}, and
  \ref{as:PL},
  with $0<\eta<1/\rho$,
  with probability at least $1-T\varphi$, we have:
\begin{align*}
&\mathbb{E}\left[
\mathcal{L}\left(\overline{\theta}^{T}\right)
-\mathcal{L}\left({\theta}^{\text{PO}}\right)
\right]\\
&\leq
\left(1-\rho\eta\right)^T
\mathbb{E}\left[
\mathcal{L}\left(\overline{\theta}^{0}\right)
-\mathcal{L}\left({\theta}^{\text{PO}}\right)
\right]\\
&\quad
+
\frac{1-\left(1-\rho\eta\right)^T}{\rho}
\Bigg[
(1-\overline{\epsilon})\overline{\pmb{\omega_D}}^2
+
\overline{\epsilon}
\left(\overline{\pmb{\omega_{D,Q}}}\right)^2-\frac{g^2_{\min}}{2}\\
&\qquad
+
L(1+\overline{\gamma})\eta
\left(
(1-\overline{\epsilon})\overline{\pmb{\omega_D}}^2
+
\overline{\epsilon}
\left(\overline{\pmb{\omega_{D,Q}}}\right)^2
+
G^2
\right)\\
&\qquad
+
4L^2(R-1)^2G^2\eta^2
\left(
(1+\overline{\gamma})^2+\sigma_\gamma^2
\right)
\Bigg].
\end{align*}
The bound has a geometrically decaying transient;
omitting the nonpositive $g_{\min}^2$ term gives a nonnegative
error-dependent neighborhood bound.
\end{theorem}

\textit{Discussion of Theorem~\ref{th:3}.} Together with
Lemmas~\ref{le:error df dfdtheta}--\ref{le:one-step-descent}, Theorem~\ref{th:3} highlights the
effects of sampling error, residual data contamination, client
heterogeneity, and the non-linearity of $f_i(\theta)$ on convergence.
When $\overline{\epsilon}>0$, the contamination-related estimation term
in Theorem~\ref{th:3} can be rewritten as
$(1-\overline{\epsilon})\overline{\pmb{\omega_D}}^{2}
+
\overline{\epsilon}
(\overline{\pmb{\omega_{D,Q}}})^{2}
=
\overline{\pmb{\omega_D}}^{2}
+
\overline{\epsilon}
\left(
(\overline{\pmb{\omega_{D,Q}}})^{2}
-
\overline{\pmb{\omega_D}}^{2}
\right).$
By definition,
$\overline{\pmb{\omega_{D,Q}}}
=
\overline{\pmb{\omega_D}}
+
LW_1(D,Q)_{\max}
+
\overline{\pmb{\omega_Q}}
\geq
\overline{\pmb{\omega_D}}.$
For fixed estimator-error bounds
$\overline{\pmb{\omega_D}}$ and
$\overline{\pmb{\omega_{D,Q}}}$, any reduction in the residual
contamination fraction $\overline{\epsilon}$ reduces this component
whenever
$LW_1(D,Q)_{\max}+\overline{\pmb{\omega_Q}}>0$.
Theorem~\ref{th:3}, however, retains the effect of any contamination
that remains in the effective sample sets and does not require complete
identification of samples drawn from $Q_i$.
The term $4L^2(R-1)^2G^2\eta^2\left((1+\overline{\gamma})^2 + \sigma^2_{\gamma}\right)$ captures client heterogeneity. Recall that $R$ is the number of local updates. In the special case where $R=1$ , meaning clients update and aggregate models at each iteration, this term becomes zero.
\begin{corollary}[Clean-data neighborhood convergence]
\label{cor:convergence without contamination}
If $\overline{\epsilon}=0$, i.e., the effective sample sets contain
no residual contamination, the bound in Theorem~\ref{th:3} is reduced
to the following:
\begin{align*}
&\mathbb{E}\left[
\mathcal{L}(\overline{\theta}^{T})
-\mathcal{L}({\theta}^{\text{PO}})
\right]\\
&\leq
(1-\rho\eta)^T
\mathbb{E}\left[
\mathcal{L}(\overline{\theta}^{0})
-\mathcal{L}({\theta}^{\text{PO}})
\right]\\
&\quad
+
\frac{1-(1-\rho\eta)^T}{\rho}
\Bigg[
\overline{\pmb{\omega_D}}^2
+
L(1+\overline{\gamma})\eta
\left(
\overline{\pmb{\omega_D}}^2+G^2
\right)\\
&\qquad
-\frac{g^2_{\min}}{2}+
4L^2(R-1)^2G^2\eta^2
\left(
(1+\overline{\gamma})^2+\sigma_\gamma^2
\right)
\Bigg].
\end{align*}
\end{corollary}

\textit{Discussion of
Corollary~\ref{cor:convergence without contamination}.}
This corollary considers the case
$\overline{\epsilon}=0$, where the effective sample sets contain no
residual contamination. In this case, the contamination-dependent terms vanish,
while $\pmb{\omega_F}$ retains the non-linearity, finite-sample,
and selection-error terms in
Lemma~\ref{le:error df dfdtheta}.
Increasing $n_i$ reduces the clean-reference sampling error;
selection error vanishes only when the effective-set estimator
coincides with its clean-reference counterpart.
A smaller $\eta$ reduces the non-linearity and explicit
local-update terms but increases the sampling term
$1/(n_i\eta^2)$; thus, $n_i$ and $\eta$ should be selected jointly,
while reducing $R$ further decreases the local-update error.

These parameter choices determine whether the descent
term dominates the estimation and local-update errors. If
$g_{\min}^2>0$ and
$\overline{\pmb{\omega_D}}^2+
L(1+\overline{\gamma})\eta
(\overline{\pmb{\omega_D}}^2+G^2)+
4L^2(R-1)^2G^2\eta^2
((1+\overline{\gamma})^2+\sigma_\gamma^2)
\leq g_{\min}^2/2$, the additive term is non-positive, and the
corollary gives linear convergence with factor $1-\rho\eta$.
Here, $g_{\min}^2$ is determined by the optimization trajectory rather
than used to select the parameters. If $b$ iterations have
$\mathbb{E}[\|g^t\|^2]=0$, the global model remains unchanged at those
iterations; if every remaining nonzero-update iteration satisfies the
per-iteration condition in Lemma~\ref{le:one-step-descent}, the
corresponding linear-convergence factor is
$(1-\rho\eta)^{T-b}$. Otherwise, the general neighborhood-convergence
result in Theorem~\ref{th:3} applies.

\subsection{Discussion}\label{subsec:discussion}
\noindent\textbf{Handling contaminated data.}
Contaminated data affects convergence
through both the residual contamination fraction
$\overline{\epsilon}$ and the discrepancy $W_1(D,Q)_{\max}$ between the
clean and contaminated distributions.
These two terms motivate the design of
Algorithm~\ref{alg:two}: reducing the residual contamination fraction
and
removing high-impact samples limits the contribution of contamination
that is farther from the clean distribution and can therefore induce a
larger discrepancy term. Thereby, the algorithm targets
contaminated gradients that exhibit a dominant directional deviation
and form a high-score tail. Contaminated samples that are not distinguished by
the SVD scores remain in the effective sample set, and their effect is
retained in the contamination-dependent terms of
Theorem~\ref{th:3}.

\noindent\textbf{Balancing Sampling Error and Computational Costs.}
 $f_{\varphi}(\varphi_i;n_i)$ depends on the distribution.
When $\epsilon_i=0$ and $\omega_{\mathrm{sel},i}=0$ (clean and unfiltered dataset)
and $d=1$, the estimator $\widehat{f}_i(\theta)$ in
distribution dynamics estimates the mean of $D_i(\theta)$. In this
case, $f_{\varphi}(\varphi_i;n_i)
=\sigma_i^2\log(2/\varphi_i)$, where $\sigma_i^2$ represents the data
variance of client $i$. In this case, with the
error bound $\Phi_i$, \textsc{ProFL} can compute $n_i$ as follows:
\[
n_i \geq
\frac{
2\left(
2\ell_{\max}^2H\|\Delta\theta\|^{-2}+F^2
\right)
\sigma_i^2\log(2/\varphi_i)
}{
2\Phi_i-
M^2\eta^2G^4H^6\|\Delta\theta\|^2\ell_{\max}^2
},
\]
For more
general filtered and contaminated case, the updated bound additionally
contains residual-contamination and data-dependent selection-error
terms. Increasing $n_i$ reduces the sampling-dependent term.
$H$, $\varphi_i$, $\|\Delta\theta\|$, and $\eta$
are known. $M$, $G$, and $F$ are estimated from prior updates. We set
$\ell_{\max}$ to the previous iteration's loss, since loss usually
decreases during training. The variance $\sigma_i^2$ is taken from
the most recent effective sample set
$\widetilde{\mathcal{S}}_i$.

\noindent\textbf{{Computational Complexity.}}
{Algorithm~\ref{alg:two} requires only the \textit{top right singular vector} of the centered gradient matrix, which can be computed using $q$ power iterations \citep{golub2013matrix}. Let $n_{i,s}$ be the number of retained samples at filtering round $s$. Each power iteration requires two matrix--vector products and therefore has complexity $\mathcal{O}(n_{i,s}d)$. Thus, computing the top singular direction and all spectral scores at round $s$ requires $\mathcal{O}(q n_{i,s}d)$ operations. In the \textit{worst case}, each nonterminating round removes only one sample, and the procedure continues until
$|\mathcal{S}'_i|\leq n_i/2$. Hence, the worst-case additional complexity is
$\mathcal{O}\!\left(
qd\sum_{s=0}^{\lceil n_i/2\rceil-1}(n_i-s)
\right)
=\mathcal{O}(q n_i^2d).$ However, in practice, the computational cost is much lower, as each round usually removes several high-score samples and the gradient-stability condition (via $J_2$) may end filtering after a few rounds. Thus, the observed cost is often closer to $\mathcal{O}(q n_i d)$ than to the worst-case bound.}

\section{Experiments}\label{sec:exp}
\subsection{Experimental Setup}
\noindent \textbf{Case Studies.} We evaluate our algorithm on five case studies covering both types of dynamics introduced in Section~\ref{sec:theorem}:
(i) \textit{Contribution dynamics}, where group contributions change dynamically; and
(ii) \textit{Distributional dynamics}, where group distributions change dynamically.
We adopt the same mixture-model formulation $D_i(\theta)$ as in the theoretical analysis.

Specifically, we study five scenarios covering both dynamics:
(1) {Dynamic Pricing with Demand Shift:} prices influence mean demands across retailers;
(2) {Dynamic Pricing with Group Contributions}---group shares of demand adjust based on remaining budgets;
(3) {Binary Strategic Classification:} clients strategically manipulate features to maximize favorable outcomes;
(4) {House Pricing Regression:} listing and selling prices interact with the model's predictions; and
(5) {Regression with Dynamic Contributions}---group fractions of retailers adapt to model prediction errors. Scenarios (2) and (5) correspond to \textit{Contribution dynamics}, while the remaining scenarios correspond to \textit{Distributional dynamics}.

\noindent \textbf{Client Configuration and Heterogeneity.}
For each experimental setting, every method is run using the same ten random seeds
$0$--$9$. We report the mean and sample standard
deviation across these runs. Table~\ref{tab:2} and Figures~\ref{fig:limited sample size},~\ref{fig:Different Sample Sizes and Learning Rates} (case study: {pricing with dynamic demands}) use 10 heterogeneous clients, except the $\alpha = 0$ case in Table~\ref{tab:2}, which uses homogeneous clients, and Figure~\ref{fig:on_server}, which uses 100 clients. {Table~\ref{tab:label_flipping_gradient} and \ref{tab:frozen_fedavg} and} Figures~\ref{fig:Binary_Classification_main} ({binary strategic classification}),~\ref{fig:Enrollment Fractions_and_Heterogeneity} ({house pricing regression}),~\ref{fig:Participation_1} ({pricing with dynamic contribution}), and~\ref{fig:Participation_4} ({regression with dynamic contribution}) all use 10 heterogeneous clients.

\noindent \textbf{{Evaluation Metrics and Reference Solutions.}}
{We use different evaluation metrics for synthetic and real-data experiments. For synthetic-data experiments, data-generation mechanisms and loss functions are fully known, so we evaluate algorithms using \textit{global performative risk} and directly compute the ground-truth of $\theta^{\mathrm{PO}}$  and $\theta^{\mathrm{PS}}$ (OPT and STAB lines). For real-data experiments (Adult and Give Me Some Credit datasets), the ground-truth performative risk and optimizer are unknown, so we do not report OPT or STAB and instead compare algorithms using \textit{accuracy}.}

\noindent \textbf{Baselines.} To date, only \citep{jin2023performative} has addressed \textit{model-dependent} distribution shifts in FL. {Thus, we compare \textsc{ProFL} with \textsc{PoFL} and \textit{Performative Federated Learning} (\textsc{PFL}). Unlike our algorithm, \textsc{PFL} converges to a performative stable solution ${\theta}^{\text{PS}}$ under stricter assumptions.}

\noindent \textbf{Dataset.} Binary classification uses both synthetic and real data (Adult~\citep{misc_adult_2} and Give Me Some Credit~\citep{GiveMeSomeCredit} datasets), while the other cases use synthetic data. All dynamics are synthetic.

Further case-study descriptions and training details are provided in Appendices~\ref{app:case studies} and~\ref{Training details}.

\subsection{Numerical Results}

\noindent \textbf{Comparison to Centralized Setting.} Table~\ref{tab:2} shows that adapting PG to the FL framework improves performance. $\gamma_i \in [\overline{\gamma} -\alpha,\overline{\gamma} +\alpha]$. Larger $\alpha$ indicates more heterogeneity. As client heterogeneity increases, the performance of both algorithms degrades. Under homogeneous distribution shifts ($\alpha = 0$), both algorithms perform similarly. \textsc{ProFL} outperforms \textsc{PG} at $\alpha = \{0.25, 0.5\}$, since the centralized method estimates only one $\frac{df}{d\theta}$ on the server and fails under variance shifts.

\begin{table}[!htbp]
\begin{center}
\fitwidth{0.48\textwidth}{!}{
\begin{tabular}{llll}
    \hline
    $\alpha$ & 0 & 0.25 & 0.5 \\
    \hline
    \textsc{PG} & \textbf{5.56 $\pm$ 0.00} & 5.67 $\pm$ 0.02 & 6.32 $\pm$ 0.06 \\
    \hline
    \textsc{ProFL} & \textbf{5.56 $\pm$ 0.00} & \textbf{5.59 $\pm$ 0.00} & \textbf{5.67 $\pm$ 0.00} \\
    \hline
\end{tabular}
}
\vspace{1mm}
\caption{Performative risk of \textsc{ProFL} (federated) and \textsc{PG} (centralized) under varying client heterogeneity ($\alpha$). \textsc{ProFL} outperforms \textsc{PG} as heterogeneity increases.}
\label{tab:2}
\end{center}
\end{table}

\begin{figure}[!htbp]
 \vspace{-0.3cm}
    \centering
    \begin{subfigure}{0.31\textwidth}
        \centering
        \includegraphics[width=\textwidth]{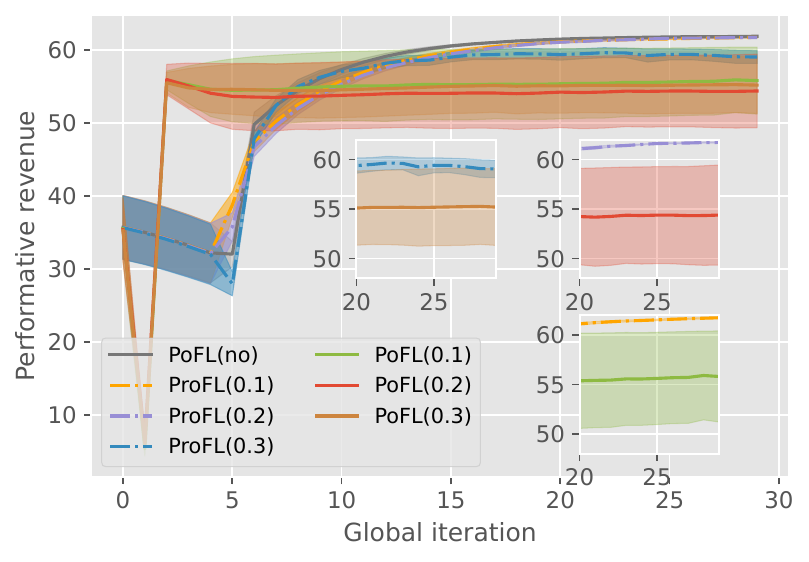}
        \caption{{Robust Gradient}}
        \label{fig:Pricing_fraction}
    \end{subfigure}\hspace{1em}%
    \begin{subfigure}{0.31\textwidth}
        \centering
        \includegraphics[width=\textwidth]{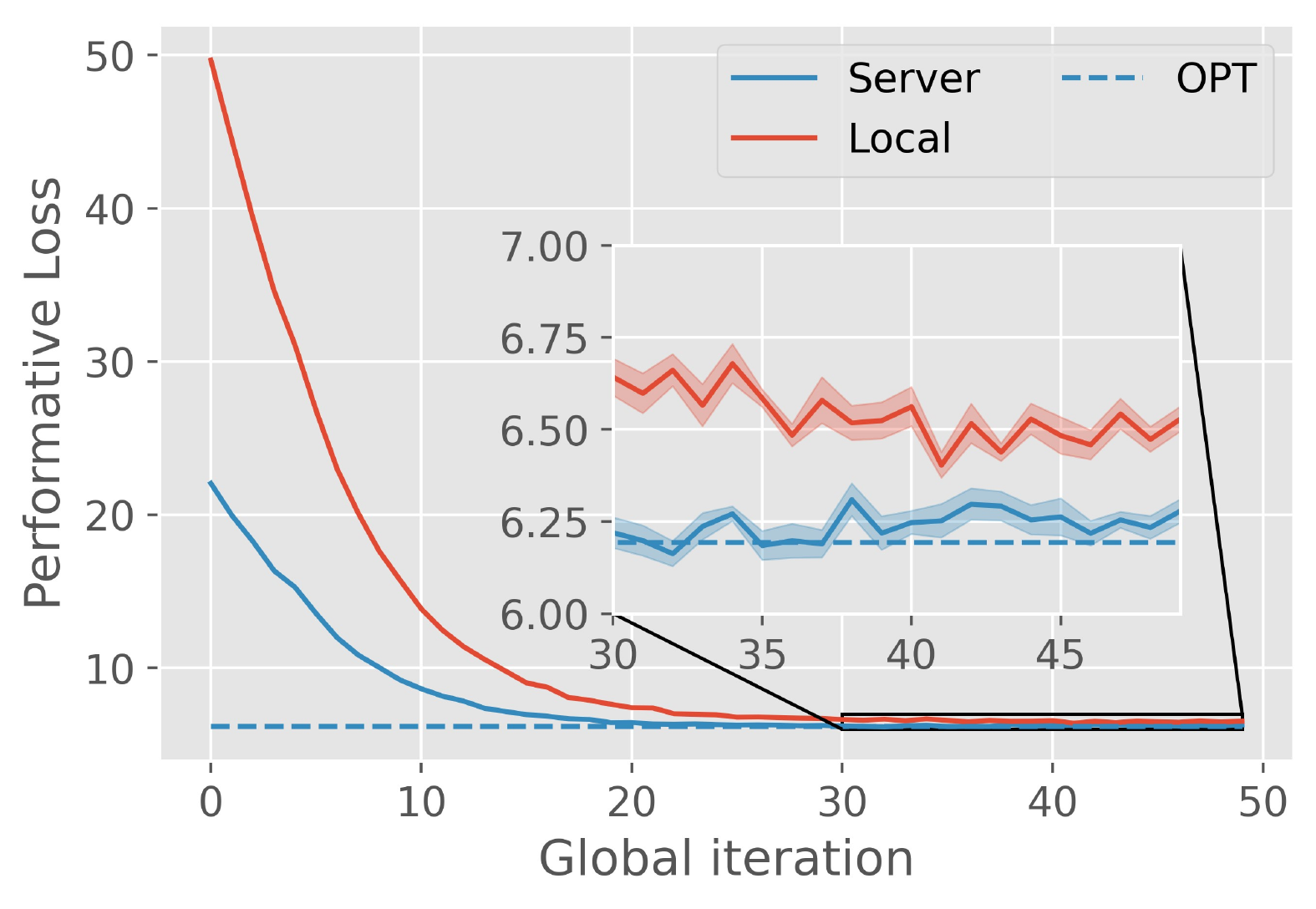}
        \caption{Estimation on Server}
        \label{fig:on_server}
    \end{subfigure}\hspace{1em}%
    \begin{subfigure}{0.31\textwidth}
        \centering
        \includegraphics[width=\textwidth]{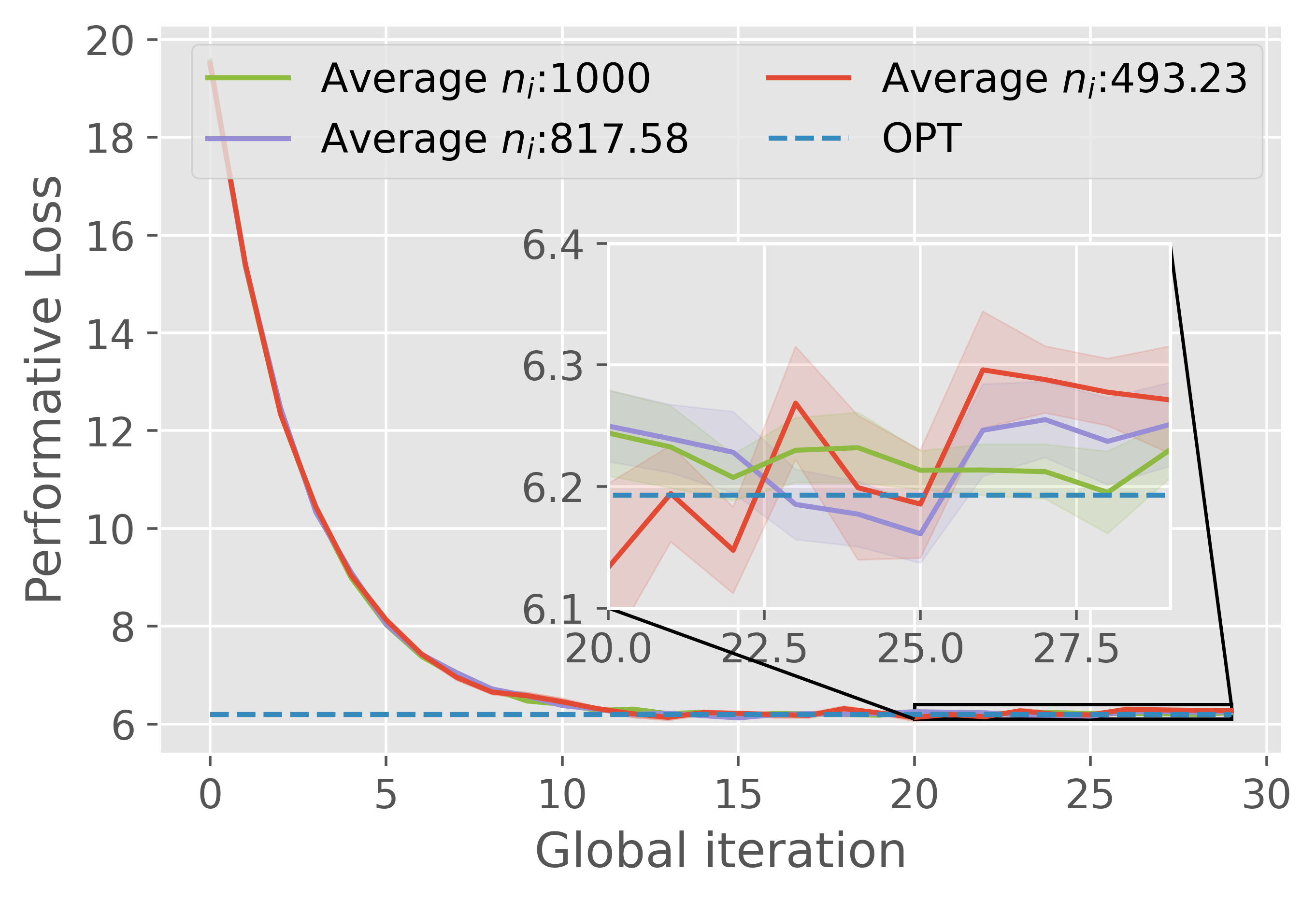}
        \caption{Adaptive Sample Size}
        \label{fig:dynamic_sample_size}
    \end{subfigure}
    \vspace{-0.1cm}
    \caption{(a) Robust gradient mitigates the impact of data contamination. (b) Estimation of $\frac{df_i}{d\theta}$ on the server outperforms local estimation in the large-client, limited-sample setting ($|\mathcal{V}| = 100$, $n_i = 60$). (c) Adaptive sampling (tolerances 0.05/0.1) reduces sample usage while maintaining accuracy.}
    \vspace{-0.1cm}
    \label{fig:limited sample size}
\end{figure}

\begin{figure}[!htbp]
 \vspace{-0.3cm}
    \centering
    \begin{subfigure}{0.31\textwidth}
        \centering
        \includegraphics[width=\textwidth]{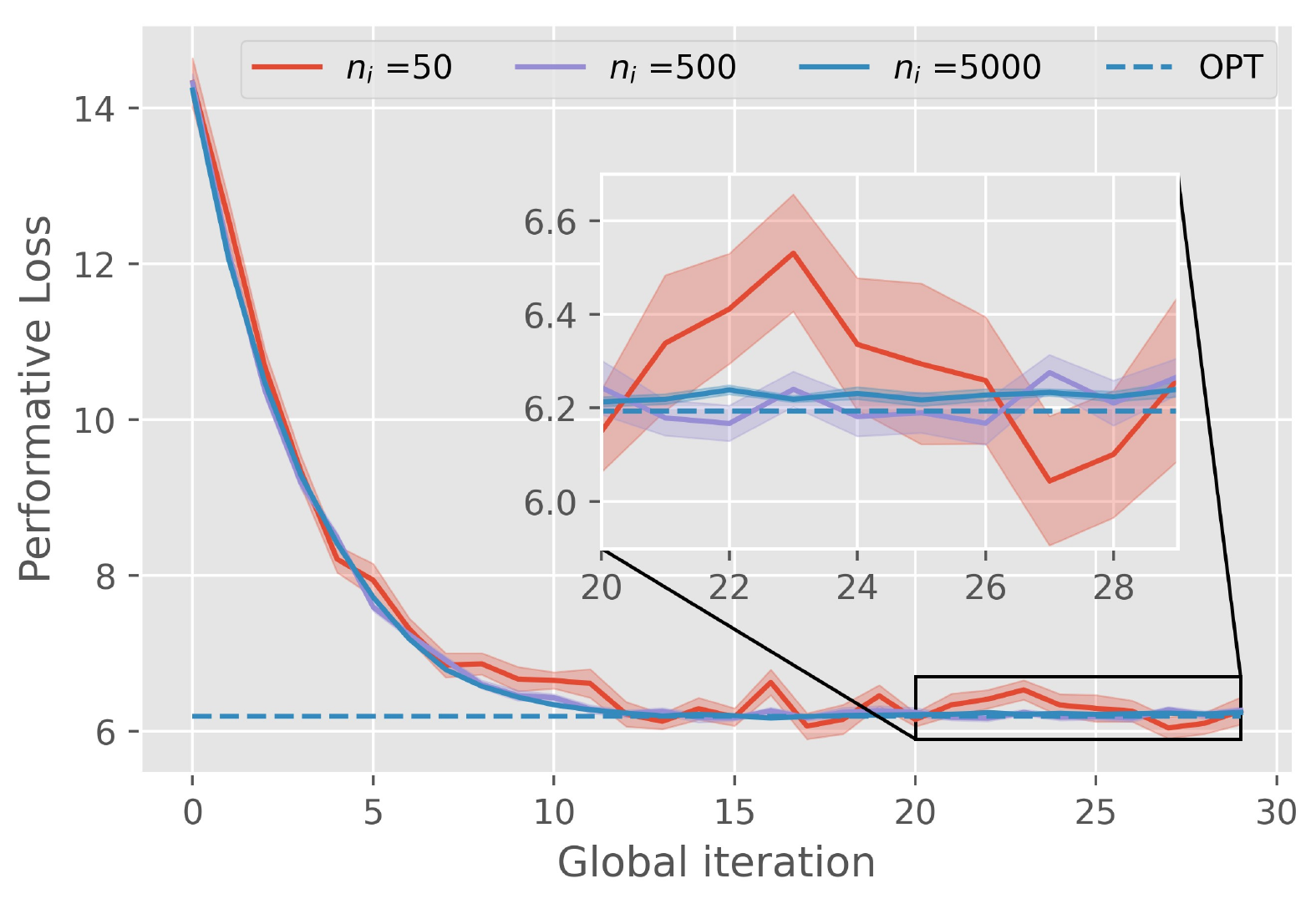}
        \caption{Different Sample Sizes}
        \label{fig:different_n}
    \end{subfigure}\hspace{1em}%
    \begin{subfigure}{0.31\textwidth}
        \centering
        \includegraphics[width=\textwidth]{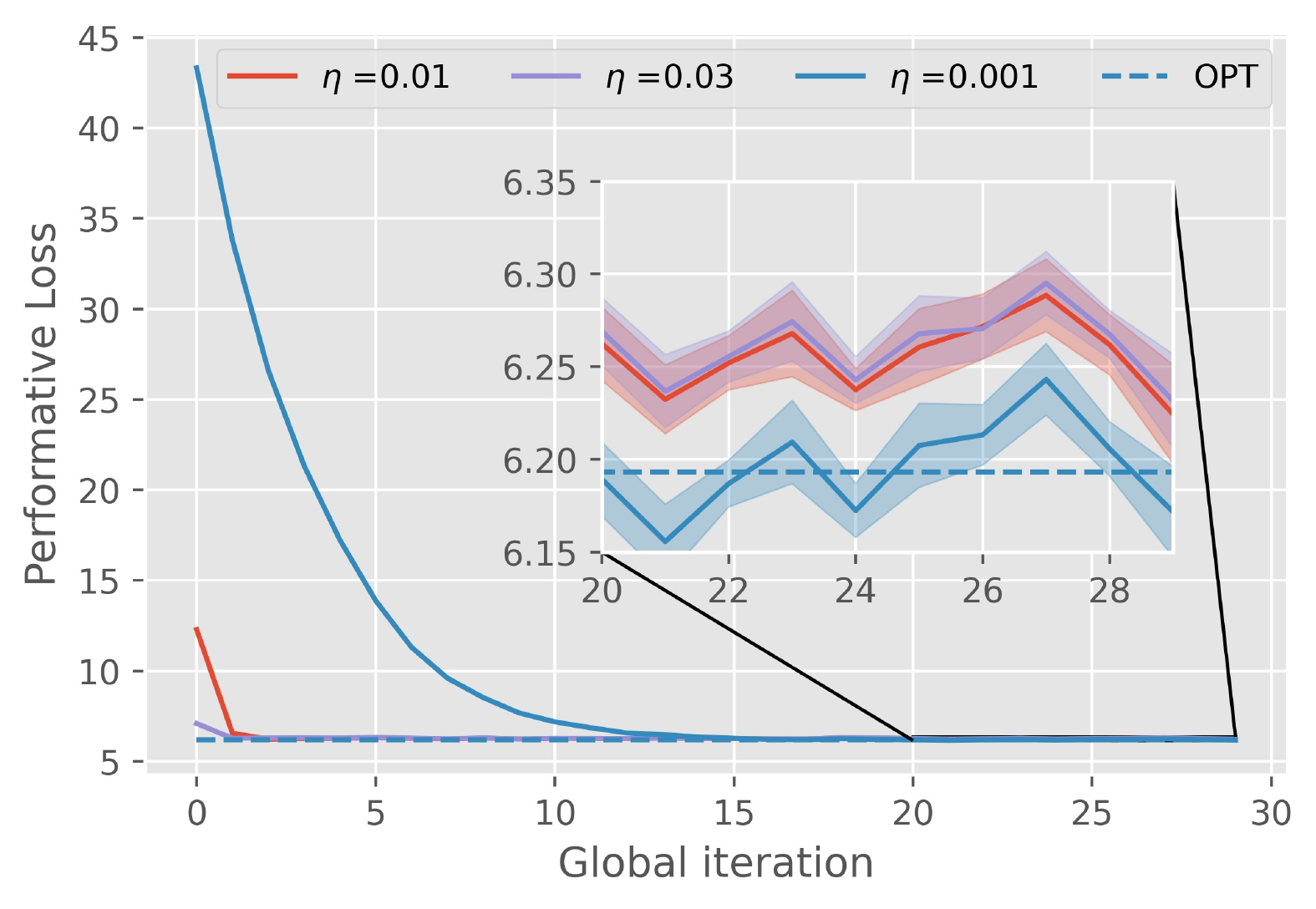}
        \caption{Different Learning Rates}
        \label{fig:different_lr}
    \end{subfigure}\hspace{1em}%
    \begin{subfigure}{0.31\textwidth}
        \centering
        \includegraphics[width=\textwidth]{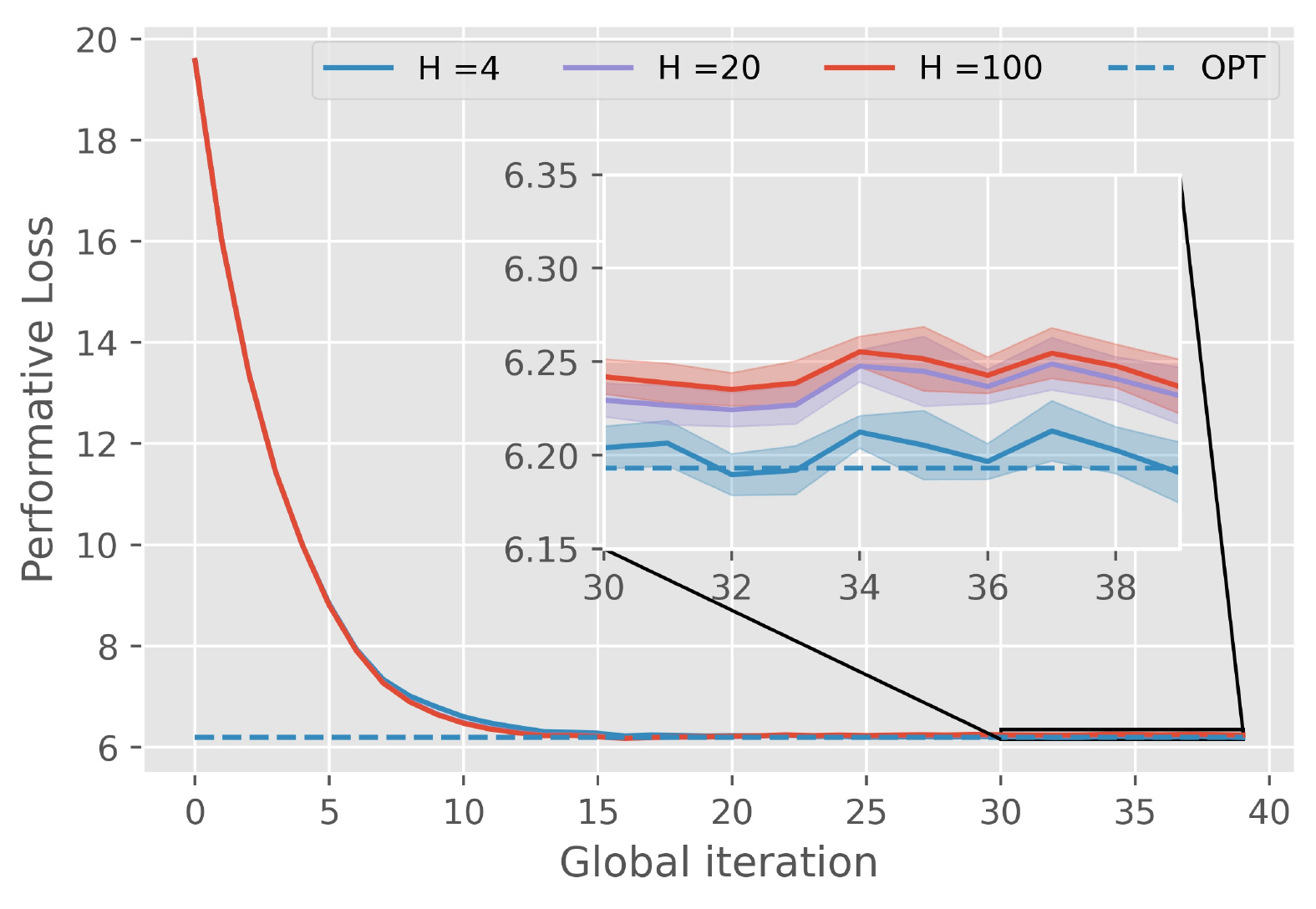}
        \caption{Different H}
        \label{fig:different_h}
    \end{subfigure}
    \vspace{-0.1cm}
    \caption{(a) Increasing sample size $n_i$ improves stability and
brings the resulting solution closer to the PO point. (b) A smaller learning rate $\eta$ slows convergence but yields solutions closer to the PO point. (c) A smaller estimation window $H$ produces more local estimates of $\frac{df_i}{d\theta}$, reducing final error at the cost of slower convergence.}
    \vspace{-0.2cm}
    \label{fig:Different Sample Sizes and Learning Rates}
\end{figure}

{\noindent \textbf{Effectiveness of \textsc{ProFL} components.}
\textit{Robust-gradient filtering.} }
Fig.~\ref{fig:Pricing_fraction} shows that contamination can
degrade performance, but the proposed \textsc{robust gradient} can
effectively reduce the impact of contaminated data when computing the
gradients i.e., the performance with the robust gradient in our method is
almost the same as the case when outliers do not exist at all. We select
\(B=50,q=20\), \(J_1=0.01\), and \(J_2=0.05\) for the robust gradient
method.

Table~\ref{tab:q_sensitivity} studies the impact of the number of
power iterations $q\in\{5,10,20,50\}$ under different \textit{original}
contamination fractions $\epsilon_i^{\mathrm{orig}}\in\{0.1,0.2,0.3\}$. All other
experimental settings are kept the same as those in
Fig.~\ref{fig:Pricing_fraction}.
For the original contamination fractions
$(0.1,0.2,0.3)$, the observed mean revenues for $q=5$ and $q=50$
differ by at most $0.61$, with no significant differences after Holm
correction ($p_{5,50}=0.721,0.721,0.859$, respectively). Thus, $q=5$
and $q=50$ show similar observed performance in the tested settings.
Overall, the leading spectral direction can be approximated using a small
fixed number of power iterations in our experiments, so a full SVD
computation is not necessary.

\begin{table}[t]
\centering
{
\begin{tabular}{c|ccc}
\hline
& \multicolumn{3}{c}{Original contamination fraction} \\
\cline{2-4}
$q$ & $0.1$ & $0.2$ & $0.3$ \\
\hline
5  & $62.87 \pm 0.19$ & \textbf{62.92 $\pm$ 0.19} & $60.23 \pm 3.38$ \\
10 & $62.89 \pm 0.16$ & $62.82 \pm 0.27$ & $60.25 \pm 2.79$ \\
20 & $62.90 \pm 0.22$ & $62.74 \pm 0.16$ & $60.38 \pm 2.75$ \\
50 & \textbf{62.97 $\pm$ 0.14} & $62.79 \pm 0.19$ & \textbf{60.84 $\pm$ 2.46} \\
\hline
$p_{5,50}$
& $0.721$
& $0.721$
& $0.859$ \\
\hline
\end{tabular}
\caption{{Performative revenue of \textsc{ProFL} with different
numbers of power iterations $q$ under different original contamination
fractions. $p_{5,50}$ denotes the Holm-adjusted $p$-value for the paired
$q=5$ versus $q=50$ comparison.}}
\label{tab:q_sensitivity}}
\end{table}

Table~\ref{tab:label_flipping_gradient} evaluates the proposed
SVD-based robust-gradient filtering method under targeted label
flipping and compares it with standard robust gradient estimators on
the Gaussian synthetic data in the binary strategic classification case
study. During gradient estimation, a fraction of positive labels is
flipped to negative labels, while the clean labels are used to compute
the reference gradient $g_{\rm clean}$. Since positive labels constitute $50\%$ of
the samples, flipping ratios $0.1$, $0.2$, $0.3$, and $0.4$ correspond
to overall contamination fractions $0.05$, $0.10$, $0.15$, and $0.20$,
respectively. We compare the proposed method
with the sample mean, coordinate-wise median, and trimmed mean. The
latter two are standard robust aggregation rules commonly used in
Byzantine-robust distributed learning~\citep{pmlr-v80-yin18a}.
Table~\ref{tab:label_flipping_gradient} reports the
gradient-estimation error $\|\hat g-g_{\rm clean}\|_2$. SVD filtering has the lowest observed mean error at
positive-label flipping ratios $0.2$, $0.3$, and $0.4$, corresponding
to overall contamination fractions $0.10$, $0.15$, and $0.20$,
respectively. The sample mean has the lowest observed mean at
the positive-label flipping ratio $0.1$ (an overall contamination
fraction of $0.05$). As the contamination level increases, the observed mean-error gap becomes increasingly
favorable to SVD filtering, which
provides evidence that SVD filtering has good performance in
the moderate-to-higher contamination setting.

Also, in a representative run with a positive-label flipping
ratio of $0.2$, the flipped samples have mean and median SVD scores of $5.7628$
and $5.1311$, respectively, both above the $90$-th percentile of the
clean-sample scores, $2.3122$. This pattern is consistent with the
high-score-tail pattern, illustrating that
the targeted pattern can arise under both Gaussian
feature contamination and targeted label flipping.

\begin{table}[t]
\centering
\fitwidth{\columnwidth}{!}{%
\begin{tabular}{c|cccc}
\hline
& \multicolumn{4}{c}{Overall Contamination Fractions} \\
\cline{2-5}
Method & $0.05$ & $0.1$ & $0.15$ & $0.2$ \\
\hline
Mean
& {$\mathbf{0.15 \pm 0.01}$}
&{$0.29 \pm 0.03$}
& {$0.44 \pm 0.04$}
& $0.59 \pm 0.06$ \\
Coordinate median
& {$0.21 \pm 0.08$}
& {$0.30 \pm 0.09$}
& {$0.39 \pm 0.13$}
& $0.52 \pm 0.19$ \\
Trimmed mean
& {$0.15 \pm 0.05$}
& {$0.26 \pm 0.09$}
& {$0.40 \pm 0.14$}
& $0.55 \pm 0.17$ \\
SVD filtering
& {$0.19 \pm 0.17$}
& {$\mathbf{0.23 \pm 0.11}$}
& {$\mathbf{0.29 \pm 0.14}$}
& $\mathbf{0.41 \pm 0.17}$ \\
\hline
$p_{\rm M}$
& $1.000$
& $0.889$
& $0.150$
& $0.094$ \\

$p_{\rm CM}$
& $1.000$
& $0.533$
& $0.281$
& $0.344$ \\

$p_{\rm TM}$
& $1.000$
& $1.000$
& $0.533$
& $0.215$ \\
\hline
\end{tabular}%
}
\caption{Gradient-estimation error
$\|\hat g-g_{\rm clean}\|_2$ under targeted label flipping.
Values are the mean $\pm$ sample standard deviation over
10 runs. The bottom panel reports two-sided exact paired sign-flip
$p$-values, Holm-adjusted across all twelve comparisons.
$p_{\rm M}$, $p_{\rm CM}$, and $p_{\rm TM}$ denote SVD filtering versus
the mean, coordinate median, and trimmed mean, respectively.}
\label{tab:label_flipping_gradient}
\end{table}

{\textit{Server-side aggregation and adaptive sample-size selection.}} Fig.~\ref{fig:on_server} shows that with a limited sample size and many clients  ($n_i = 60, |\mathcal{V}| = 100$), estimating $\widehat{\frac{df_i}{d\theta}}$ on the server side results in a lower loss value compared to local estimation and approaches the optimal value. Fig.~\ref{fig:dynamic_sample_size} demonstrates the effectiveness of adaptive sample sizing. We tested two tolerances (0.05 for purple, 0.1 for red), with each client processing up to 1000 samples per iteration. The adaptive approach reduces the sample size while achieving similar results within the error tolerance, leading to substantial computational savings. Both experiments have non-linear $f_i(\theta)$.

\noindent \textit{Impact of hyperparameters.}
We examine the effects of sample size $n_i$, learning rate $\eta$, and estimation window $H$. Fig.~\ref{fig:different_n} shows that increasing $n_i$ improves stability and brings the algorithm closer to the optimal point. Fig.~\ref{fig:different_lr} indicates that a smaller $\eta$ slows convergence but moves $\theta$ nearer to the optimum. Fig.~\ref{fig:different_h} demonstrates that a smaller $H$ results in {estimates} of $\frac{df_i}{d\theta}$ based on closer but smaller groups of $\widehat{f_i}$ and $\theta_i$, resulting in convergence closer to the performative optimum, albeit at a slightly slower speed.

\begin{figure}[!htbp]
\vspace{-0.2cm}
\centering
\begin{subfigure}{0.23\textwidth}
    \centering
    \includegraphics[trim=0.21cm 0.35cm 0.25cm 0.2cm,clip,width=\textwidth]{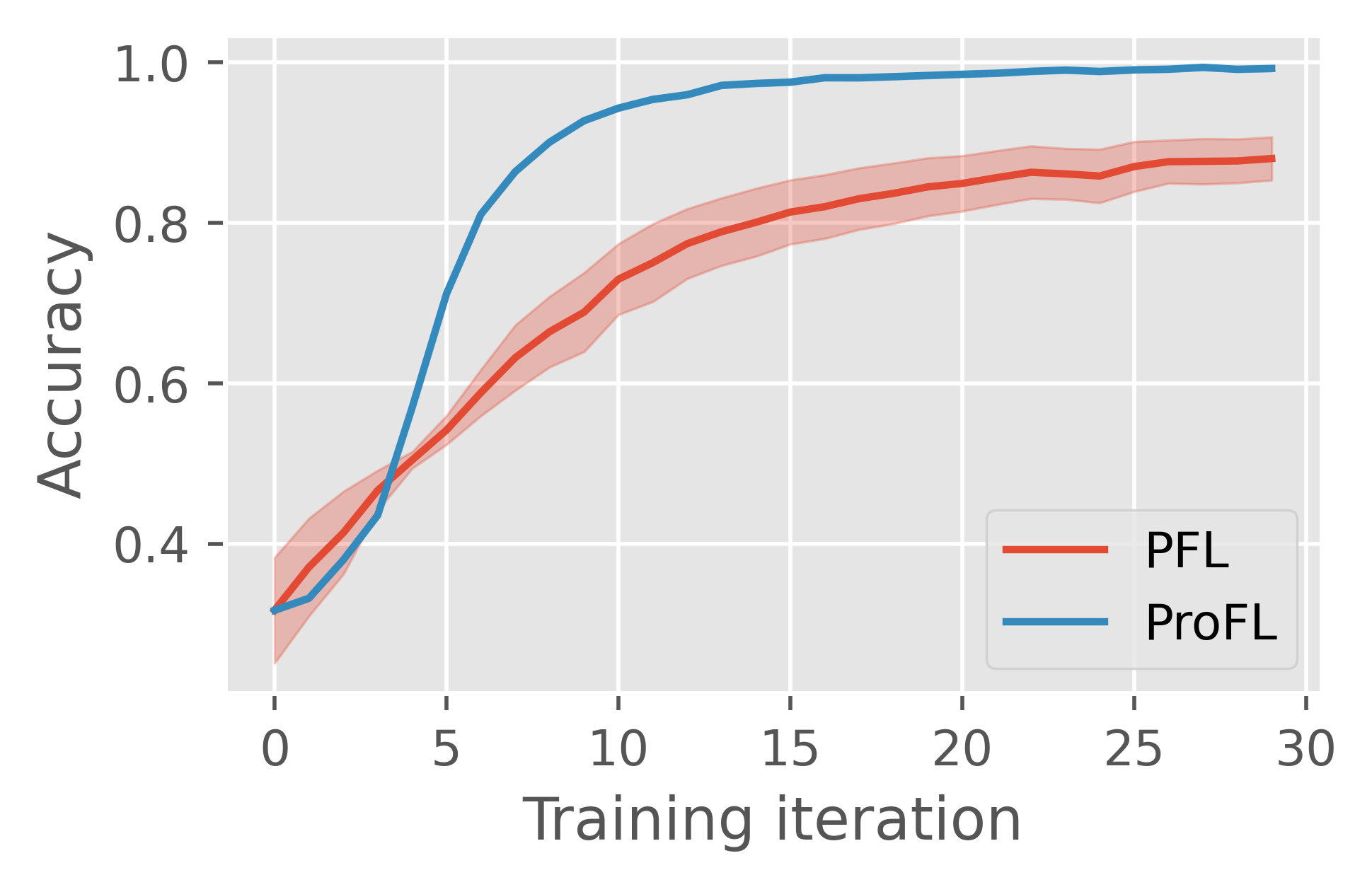}
    \caption{Same Distributions}
    \label{fig:Binary Classification}
\end{subfigure}\hspace{0.5em}%
\begin{subfigure}{0.23\textwidth}
    \centering
    \includegraphics[trim=0.21cm 0.35cm 0.25cm 0.2cm,clip,width=\textwidth]{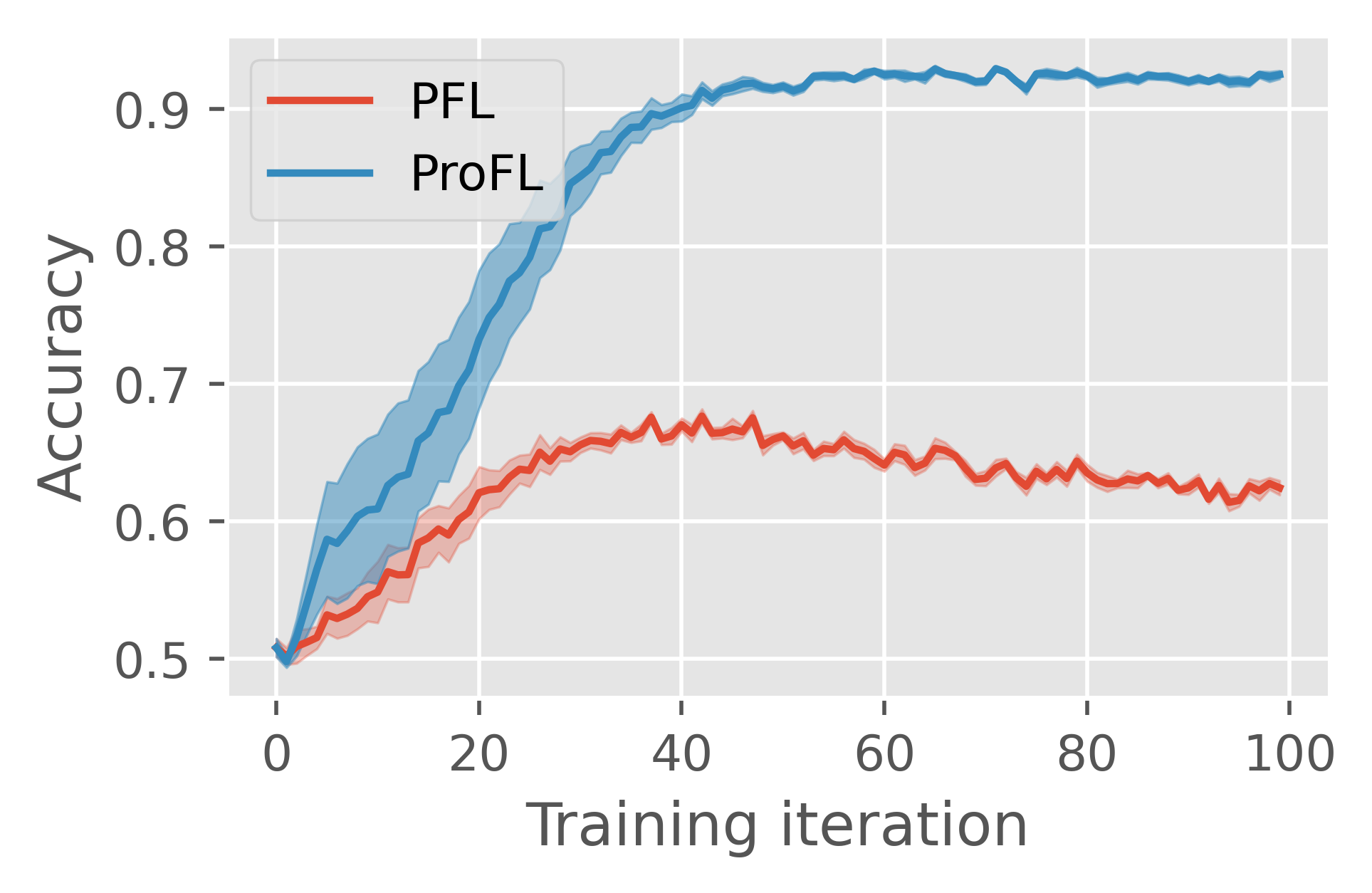}
    \caption{Different Distributions}
    \label{fig:Binary Classification_d3}
\end{subfigure}\hspace{0.5em}%
\begin{subfigure}{0.23\textwidth}
    \centering
    \includegraphics[width=\textwidth]{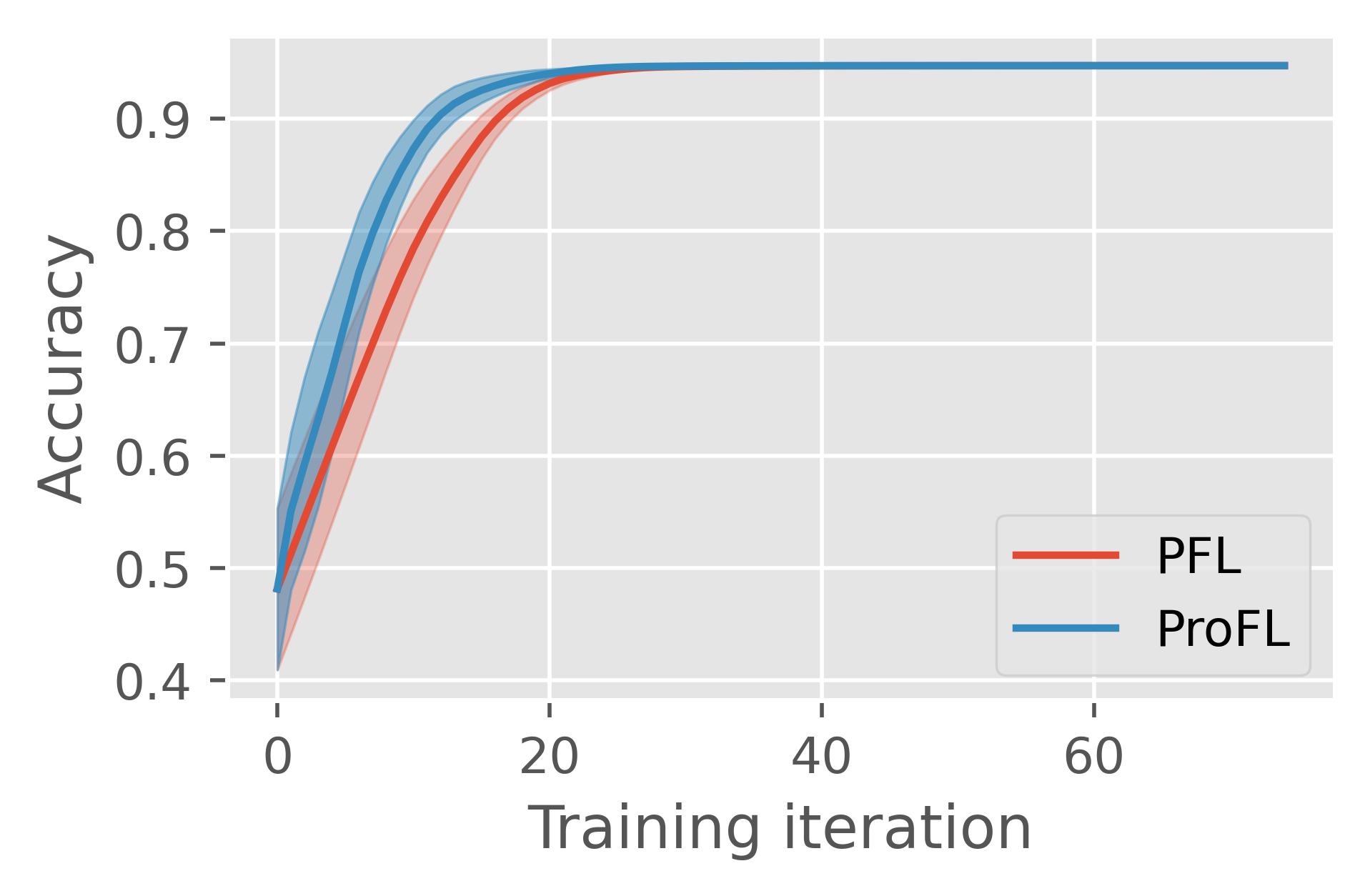}
    \caption{Give Me Some Credit}
    \label{fig:Binary Classification_Cridit}
\end{subfigure}\hspace{0.5em}%
\begin{subfigure}{0.23\textwidth}
    \centering
    \includegraphics[width=\textwidth]{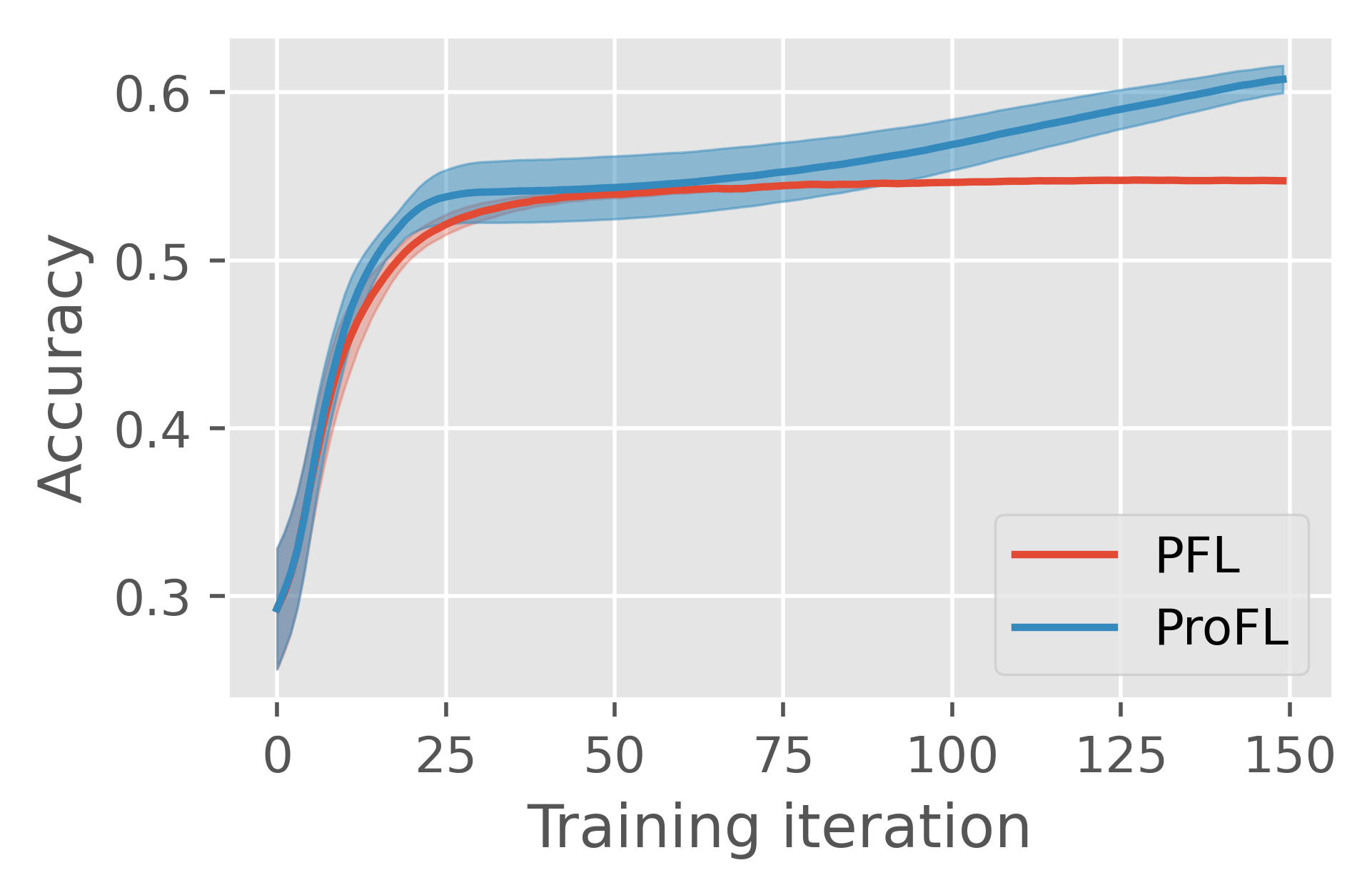}
    \caption{figures/Adult.png}
    \label{fig:Adult}
\end{subfigure}
\vspace{-0.15cm}
\caption{Convergence comparison between \textsc{ProFL} and \textsc{PFL} on both synthetic and real-world datasets. \textsc{ProFL} consistently achieves higher accuracy with fewer communication rounds.}
\vspace{-0.15cm}
\label{fig:Binary_Classification_main}
\end{figure}

\noindent
\textbf{Performance relative to the performative optimal
solution.}
\textsc{PFL} converges to the PS point $\theta^{\text{PS}}$, which can be suboptimal, whereas \textsc{ProFL} consistently converges closer to the PO point $\theta^{\text{PO}}$. We conduct several experiments across diverse case studies to support this claim. On Gaussian synthetic data, \textsc{ProFL} demonstrates
higher accuracy and greater stability than \textsc{PFL} (Figs.~\ref{fig:Binary Classification},~\ref{fig:Binary Classification_d3}), which converges to $\theta^{\text{PS}}$. On real datasets (Figs.~\ref{fig:Binary Classification_Cridit},~\ref{fig:Adult}), \textsc{ProFL} achieves faster convergence and higher final accuracy (e.g., $60.78\%$ vs. $54.74\%$ on \textit{figures/Adult.png}). Since \textsc{PFL} is sensitive to outliers, we set $\epsilon_i^{\mathrm{orig}}=0$ for fair comparison. In the non-linear $f_i(\cdot)$ cases (Fig.~\ref{fig:Participation_1} and~\ref{fig:Participation_4}), \textsc{ProFL} again converges faster and reaches a point much closer to $\theta^{\text{PO}}$, while \textsc{PFL} remains at the suboptimal $\theta^{\text{PS}}$. Additionally, in Fig.~\ref{fig:Participation_4}, \textsc{ProFL} demonstrates more stable convergence despite a persistent gap to $\theta^{\text{PO}}$ caused by the non-linearity of $f_i(\cdot)$. Strategies to mitigate this gap are discussed in Subsection~\ref{subsec:discussion}.
\begin{figure}[!htbp]
    \centering
    \begin{subfigure}{0.23\textwidth}
        \centering
        \includegraphics[width=\textwidth]{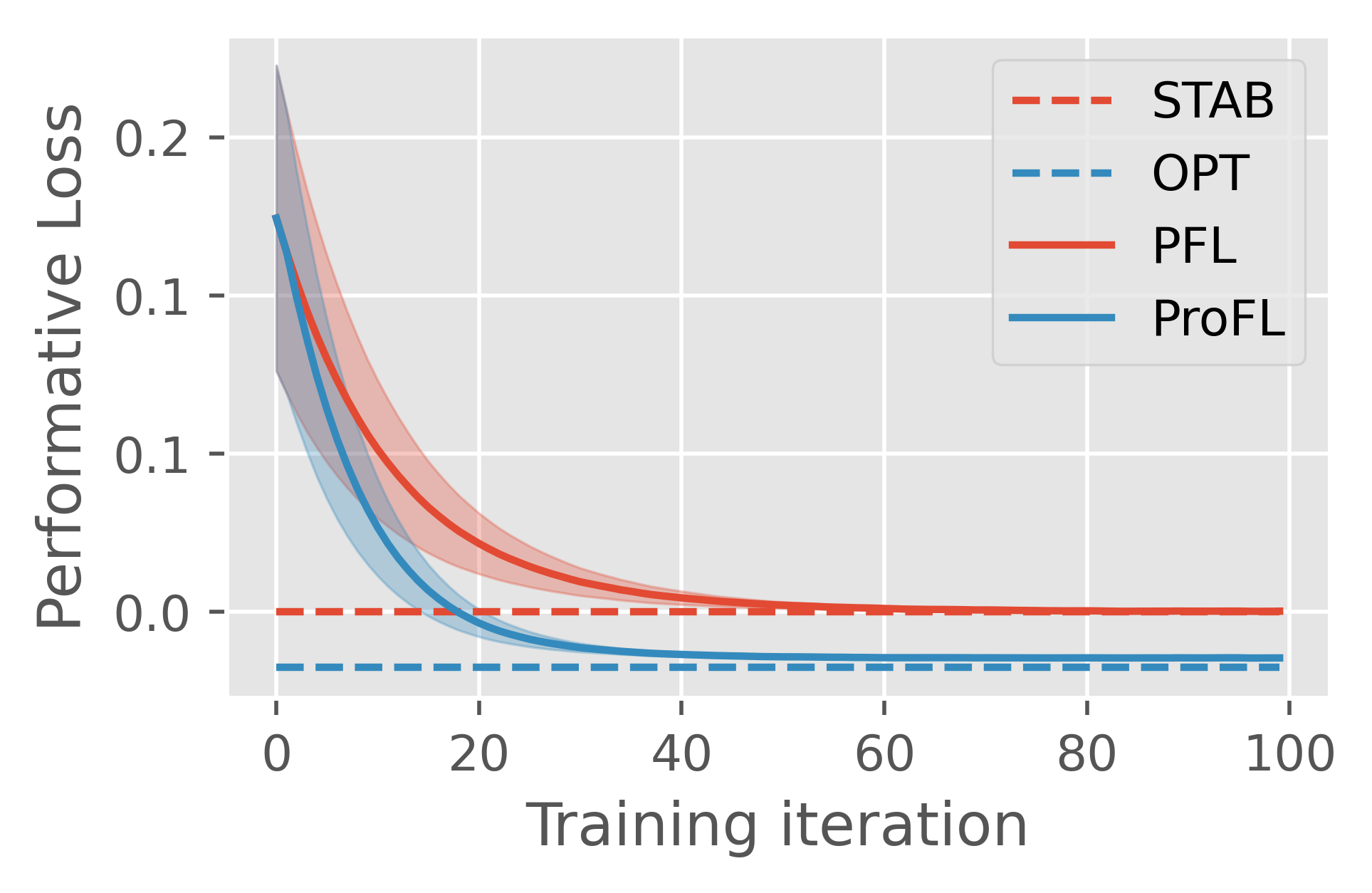}
        \caption{Pricing}
        \label{fig:Participation_1}
    \end{subfigure}\hspace{0.5em}%
    \begin{subfigure}{0.23\textwidth}
        \centering
        \includegraphics[width=\textwidth]{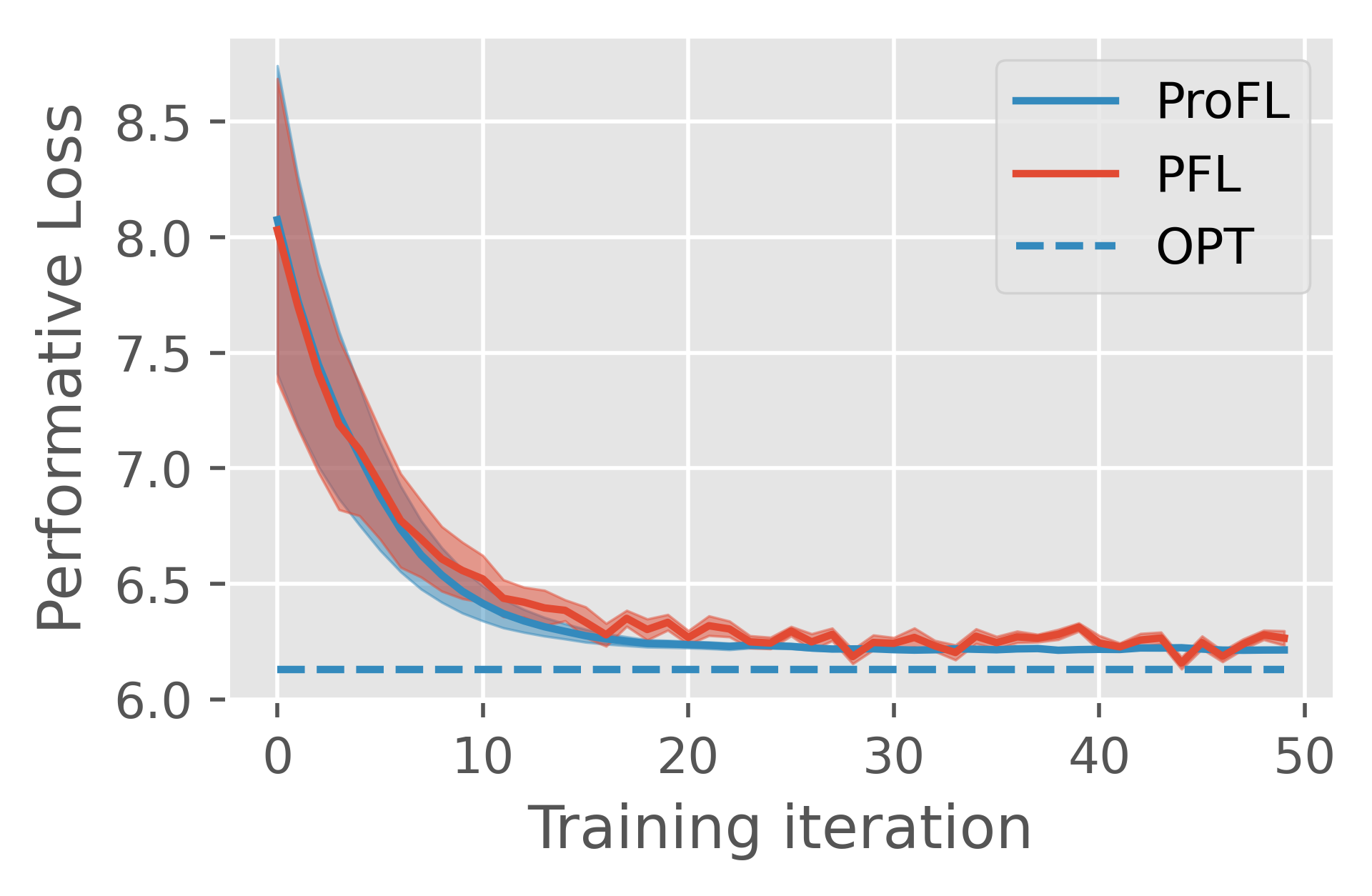}
        \caption{Regression}
        \label{fig:Participation_4}
    \end{subfigure}
    \caption{Convergence for \textit{pricing with dynamic contribution} and \textit{regression with dynamic contribution}. \textsc{ProFL} approaches the performative optimum more closely than \textsc{PFL}, with faster convergence and reduced oscillation, despite a gap due to non-linear $f_i(\cdot)$.}
    \label{fig:Participation Dynamics}
\end{figure}

\begin{figure}[!htbp]
 \vspace{-0.3cm}
    \centering
    \begin{subfigure}{0.23\textwidth}
        \centering
        \includegraphics[width=\textwidth]{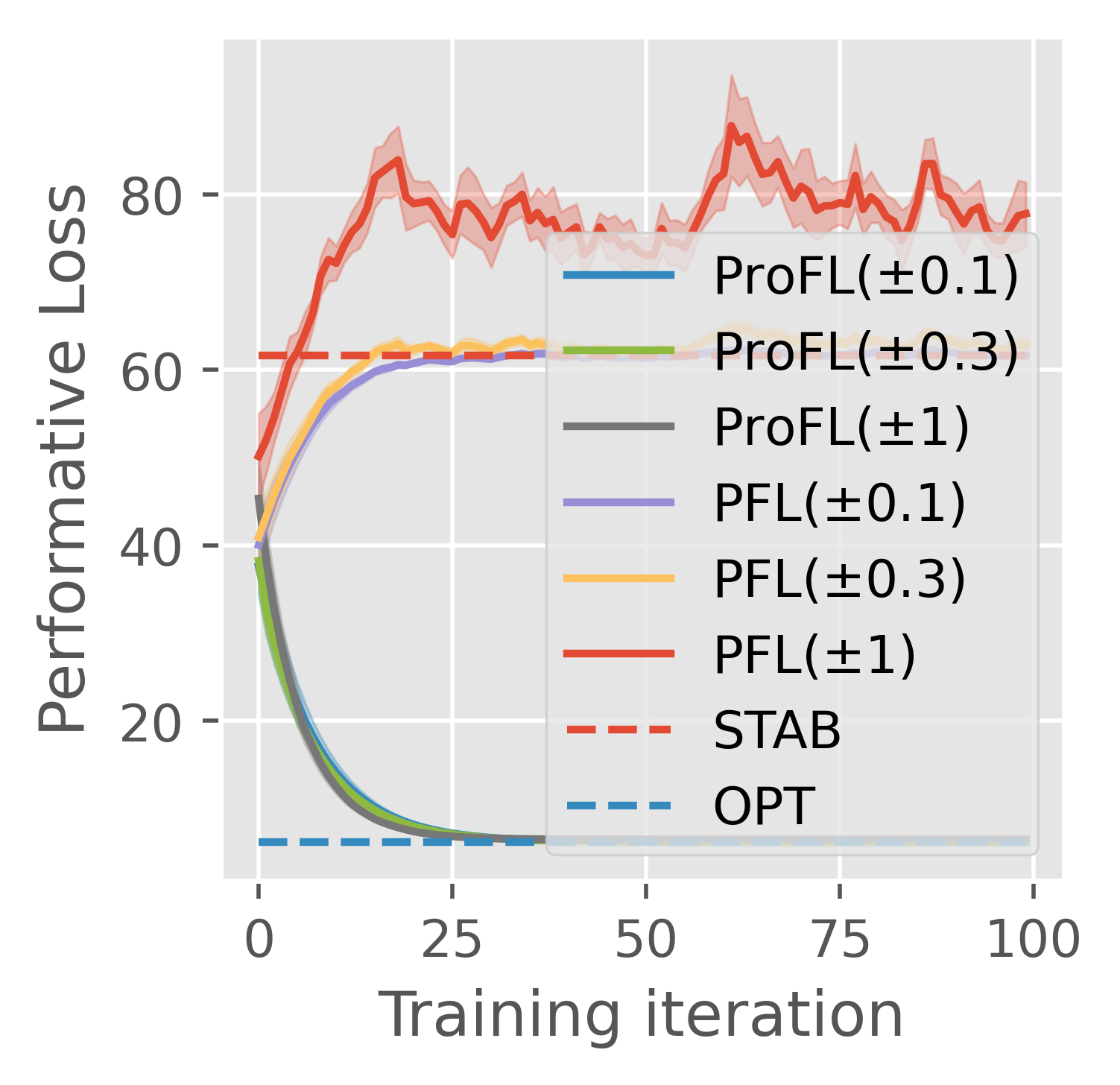}
        \caption{Different Heterogeneity}
        \label{fig:Heterogeneity}
    \end{subfigure}
    \begin{subfigure}{0.23\textwidth}
        \centering
        \includegraphics[width=\textwidth]{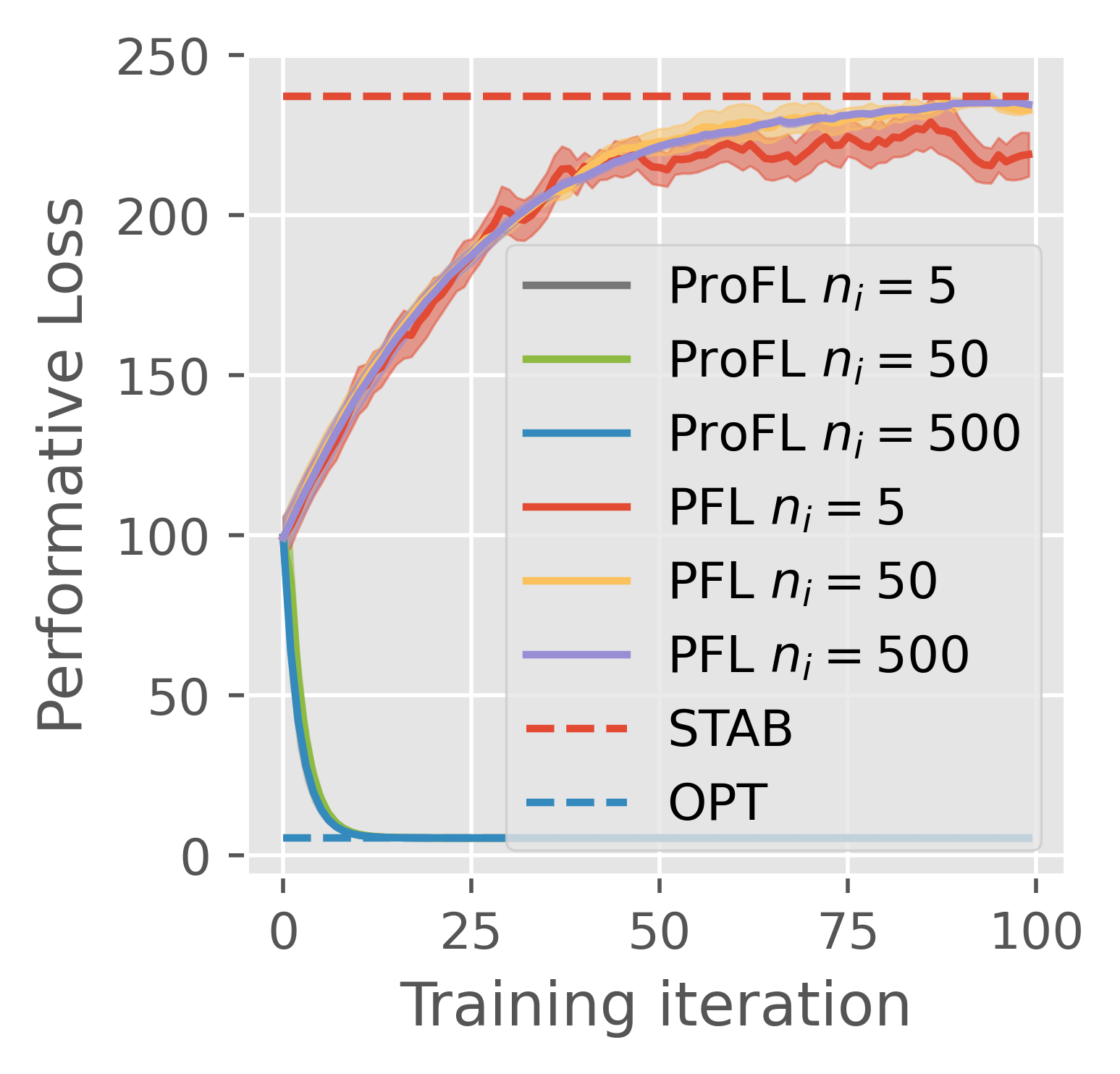}
        \caption{Sample Size}
        \label{fig:Sample_size}
    \end{subfigure}\hspace{0.5em}%
    \caption{Impact of (a) heterogeneity and (b) sample size on convergence. \textsc{ProFL} approaches the PO point more closely than
\textsc{PFL}.}
    \label{fig:Enrollment Fractions_and_Heterogeneity}
\end{figure}
In Fig.~\ref{fig:Enrollment Fractions_and_Heterogeneity},
\textsc{ProFL} approaches the PO solution more
closely than \textsc{PFL}. Across all settings, \textsc{PFL} either
stagnates at $\theta^{\text{PS}}$ or exhibits oscillations under complex
dynamics. Fig.~\ref{fig:Heterogeneity} illustrates that \textsc{ProFL}
remains close to $\theta^{\text{PO}}$ across increasing
heterogeneity ($\alpha\in\{0.1,0.3,1\}$), and
Fig.~\ref{fig:Sample_size} shows that increasing sample size improves
the stability of both algorithms, although
\textsc{ProFL} approaches the PO solution more
closely.

\noindent{\textbf{Comparison with non-performative FedAvg baselines.}
We compare \textsc{ProFL} with two non-performative FedAvg baselines, Frozen FedAvg-Static and Frozen FedAvg-Initial-Induced, on the Adult dataset. Frozen FedAvg-Static uses the original fixed Adult dataset. Frozen FedAvg-Initial-Induced first constructs the initial induced
distribution $D(\overline{\theta}^0)$ and then trains FedAvg on this fixed distribution. Both baselines are evaluated on the distribution induced by their final deployed models. As shown in
Table~\ref{tab:frozen_fedavg}, both frozen FedAvg baselines achieve lower final accuracy than \textsc{PFL} and \textsc{ProFL}. \textsc{PFL} achieves the second-highest accuracy because it adapts to the currently induced distribution, although it ignores the distribution-derivative component. \textsc{ProFL} achieves the highest final accuracy by incorporating both components of the PG.}

\begin{table}[t]
\centering
{
\begin{tabular}{lc}
\hline
Method & Final accuracy (\%) \\
\hline
Frozen FedAvg-Static & $36.68 \pm 0.49$ \\
Frozen FedAvg-Initial-Induced & $45.63 \pm 3.47$ \\
\textsc{PFL} & $54.74 \pm 0.06$ \\
\textsc{ProFL} & $\mathbf{60.78 \pm 2.57}$ \\
\hline
\end{tabular}
\caption{{Comparison with non-performative FedAvg baselines on the Adult dataset. All methods are evaluated on the distributions induced by their final deployed models.}}
\label{tab:frozen_fedavg}
}
\end{table}

\section{Conclusion and Future Work}\label{sec:conclusion}

We proposed \textsc{ProFL},
the first approach to directly target the PO point in an
FL setting, and analyzed the impact of data contamination,
addressing key challenges in FL: limited local data
sample sizes, constrained local computation, and the potentially high
communication cost associated with multiple rounds of model exchange.
We provided a neighborhood-convergence analysis under
the PL condition, which is weaker than the assumptions required by
\textsc{PFL}.

However, several limitations remain. Our experiments rely on synthetic distribution shifts due to the lack of real-world datasets with dynamic performative feedback. Future work includes exploring real datasets for more realistic evaluation and relaxing the PL condition to cover broader models such as neural networks, especially large language models.

\FloatBarrier
\appendix
\section{Additional Proofs not Detailed in the Main Paper}\label{ap:extral proofs}

\subsection{Proof of Theorem \ref{th:3}}\label{pr:convergence}
\begin{proof}
By Assumption~\ref{as:smoothness}, each local performative risk
has an $L(1+\gamma_i)$-Lipschitz gradient. Therefore,
since $\mathcal{L}(\theta)=\sum_{i=1}^N\alpha_i\mathcal{L}_i(\theta), \|\nabla\mathcal{L}(\theta)-\nabla\mathcal{L}(\theta')\|
\leq
L(1+\overline{\gamma})\|\theta-\theta'\|,$
where
$\overline{\gamma}=\sum_{i=1}^N\alpha_i\gamma_i$. Let $\overline{\theta}^{t+1} - \overline{\theta}^{t} = -g^t \eta$ and $g^t = \sum_{i=1}^N \alpha_i g^t_i$. We obtain
\begin{align}
    \mathbb{E}[\mathcal{L}(\overline{\theta}^{t+1})]
&\leq \mathbb{E}[\mathcal{L}(\overline{\theta}^{t})]
\underbrace{- \eta \mathbb{E}\left[\langle \nabla \mathcal{L}(\overline{\theta}^{t}), g^t \rangle\right]}_{T_1} \notag \\
&\quad + \frac{L\eta^2}{2} \sum_{i=1}^N \alpha_i(1+\gamma_i)
\underbrace{\mathbb{E}\left[\|g^t\|^2\right]}_{T_2}.\label{eq:upper bound}
\end{align}
Using Jensen's inequality and Lemma~\ref{le:upper bounds}, which gives $\left\|\sum_{i=1}^N \alpha_i \nabla \mathcal{L}_i(\theta_i^t)\right\|^2 \leq G^2$, we obtain the upper bound of $T_2$:
\begin{align*}
    \mathbb{E}\left[\|g^t\|^2\right]
    &= \mathbb{E}\left[\left\| \sum^N_{i=1} \alpha_i \left( g^{t}_i - \nabla \mathcal{L}_i(\theta^{t}_i) + \nabla \mathcal{L}_i(\theta^{t}_i)\right)\right\|^2\right]\\
     &\leq  2\mathbb{E}\left[\left\| \sum^N_{i=1} \alpha_i \left( g^{t}_i - \nabla \mathcal{L}_i(\theta^{t}_i)\right)\right\|^2\right] + 2G^2.
\end{align*}
We begin by bounding $T_1$. Starting from $-\mathbb{E}\left[\langle \nabla \mathcal{L}(\overline{\theta}^t), g^t \rangle\right] = -\tfrac{1}{2} \mathbb{E}\left[\|\nabla \mathcal{L}(\overline{\theta}^t)\|^2 + \|g^t\|^2\right]
    + \tfrac{1}{2} \mathbb{E}\left[\|\nabla \mathcal{L}(\overline{\theta}^t) - g^t\|^2\right]$. We have $-\tfrac{1}{2} \mathbb{E}\left[\|\nabla \mathcal{L}(\overline{\theta}^t)\|^2 \right] \leq -\rho \mathbb{E}\left[\mathcal{L}(\overline{\theta}^t)
    - \mathcal{L}({\theta}^{\text{PO}})\right]$ by {the PL condition on the global performative risk
$\mathcal{L}(\theta)$ (Assumption~\ref{as:PL}).}
    Applying Jensen's inequality and $\nabla \mathcal{L}(\overline{\theta}^t) - g^t = \nabla \mathcal{L}(\overline{\theta}^t)
    - \sum_{i=1}^N \alpha_i \nabla \mathcal{L}_i(\theta_i^t)
    + \sum_{i=1}^N \alpha_i \nabla \mathcal{L}_i(\theta_i^t) - g^t$, we obtain
     {\small
\begin{align*}
    \tfrac{1}{2} \mathbb{E}\left[\|\nabla \mathcal{L}(\overline{\theta}^t) - g^t\|^2\right]
        & \leq \mathbb{E}\left[\left\| \sum_{i=1}^N \alpha_i
    \left(\nabla \mathcal{L}_i(\overline{\theta}^t)
    - \nabla \mathcal{L}_i(\theta_i^t)\right)\right\|^2\right] \\&+\mathbb{E}\left[\left\| \sum_{i=1}^N \alpha_i \nabla \mathcal{L}_i(\theta_i^t)
    - g^t \right\|^2\right].
\end{align*}
}
 Using Assumption~\ref{as:smoothness} and Lemma~\ref{le:upper bound variance}, we have
    {\small
       \begin{align*}
& \mathbb{E}\left[\left\| \sum_{i=1}^N \alpha_i
    \left(\nabla \mathcal{L}_i(\overline{\theta}^t)
    - \nabla \mathcal{L}_i(\theta_i^t)\right)\right\|^2 \right] \\
    &\leq L^2 \mathbb{E}\left[\sum_{i=1}^N \alpha_i (1+\gamma_i)^2
    \|\overline{\theta}^t - \theta_i^t\|^2\right] \\
    & \leq 4L^2(R{-}1)^2 G^2 \eta^2 \left((1+\overline{\gamma})^2 + \sigma_{\gamma}^2\right).
    \end{align*}
    }

Then we have the upper bound of $T_1$:
\begin{align*}
    T_1
&\leq -\rho \eta\mathbb{E}\left[\mathcal{L}(\overline{\theta}^t)
    - \mathcal{L}({\theta}^{\text{PO}})\right]
    - \tfrac{\eta}{2} \mathbb{E}[\|g^t\|^2] \\
&\quad + 4L^2(R{-}1)^2 G^2 \eta^3 \left((1+\overline{\gamma})^2 + \sigma_{\gamma}^2\right) \\
&\quad + \eta\mathbb{E}\left[\left\| \sum_{i=1}^N \alpha_i
    \nabla \mathcal{L}_i(\theta_i^t) - g^t \right\|^2\right].
\end{align*}

 Finally, substituting $T_1$ and $T_2$ into
Eq.~\eqref{eq:upper bound}, and letting
$\overline{\gamma}=\sum_{i=1}^N\alpha_i\gamma_i$ and
$\sigma_\gamma^2=\sum_{i=1}^N
\alpha_i(\overline{\gamma}-\gamma_i)^2$, we obtain
\begin{equation}
\begin{aligned}
&\mathbb{E}\left[
\mathcal{L}(\overline{\theta}^{t+1})
-\mathcal{L}(\theta^{\mathrm{PO}})
\right]\\
&\leq
(1-\rho\eta)
\mathbb{E}\left[
\mathcal{L}(\overline{\theta}^{t})
-\mathcal{L}(\theta^{\mathrm{PO}})
\right]\\
&\quad+
4L^2(R-1)^2G^2\eta^3
\left(
(1+\overline{\gamma})^2+\sigma_\gamma^2
\right)\\
&\quad+
LG^2\eta^2(1+\overline{\gamma})\\
&\quad+
L\eta^2(1+\overline{\gamma})
\mathbb{E}\left[
\left\|
\sum_{i=1}^N\alpha_i
\nabla\mathcal{L}_i(\theta_i^t)-g^t
\right\|^2
\right]\\
&\quad+
\eta\mathbb{E}\left[
\left\|
\sum_{i=1}^N\alpha_i
\nabla\mathcal{L}_i(\theta_i^t)-g^t
\right\|^2
-\frac{1}{2}\|g^t\|^2
\right].
\end{aligned}
\label{eq:one_step_descent}
\end{equation}

Denote
$\mathbb{E}[
\mathcal{L}(\overline{\theta}^{t})
-\mathcal{L}(\theta^{\mathrm{PO}})]$
by $a_t$. Then,
{\small
\begin{align*}
a_{t+1}
&\leq
(1-\rho\eta)a_t\\
&\quad+
4L^2(R-1)^2G^2\eta^3
\left(
(1+\overline{\gamma})^2+\sigma_\gamma^2
\right)\\
&\quad+
LG^2\eta^2(1+\overline{\gamma})\\
&\quad+
L\eta^2(1+\overline{\gamma})
\mathbb{E}\left[
\left\|
\sum_{i=1}^N\alpha_i
\nabla\mathcal{L}_i(\theta_i^t)-g^t
\right\|^2
\right]\\
&\quad+
\eta\mathbb{E}\left[
\left\|
\sum_{i=1}^N\alpha_i
\nabla\mathcal{L}_i(\theta_i^t)-g^t
\right\|^2
-\frac{1}{2}\|g^t\|^2
\right]\\
&=
(1-\rho\eta)a_t
+T_4\eta
+T_3\eta^2,
\end{align*}
}
where
\begin{align*}
T_4
&=
\mathbb{E}\left[
\left\|
\sum_{i=1}^N\alpha_i
\nabla\mathcal{L}_i(\theta_i^t)-g^t
\right\|^2
-\frac{1}{2}\|g^t\|^2
\right],
\end{align*}
\begin{align*}
T_3
&=
4L^2(R-1)^2G^2\eta
\left(
(1+\overline{\gamma})^2+\sigma_\gamma^2
\right)\\
&\quad+
LG^2(1+\overline{\gamma})\\
&\quad+
L(1+\overline{\gamma})
\mathbb{E}\left[
\left\|
\sum_{i=1}^N\alpha_i
\nabla\mathcal{L}_i(\theta_i^t)-g^t
\right\|^2
\right].
\end{align*}

Denote
$\mathbb{E}[\mathcal{L}(\overline{\theta}^{t})
-\mathcal{L}(\theta^{\mathrm{PO}})]$ by $a_t$.
Let $\overline{\epsilon}^{\,t}
:=
\sum_{i=1}^N\alpha_i\mathbb{E}[\epsilon_i^t],$ and $\overline{\epsilon}
:=
\max_{0\leq t<T}\overline{\epsilon}^{\,t},$
and define
$\overline{\pmb{\omega_{D,Q}}}
:=
\overline{\pmb{\omega_D}}
+LW_1(D,Q)_{\max}
+\overline{\pmb{\omega_Q}}.$
Here, $\overline{\pmb{\omega_D}}$ and
$\overline{\pmb{\omega_Q}}$ are uniform upper bounds over all clients
and iterations and include the updated effective-set estimation errors. By Jensen's inequality and
Lemma~\ref{le:upper bound L2 with h},
\begin{align*}
\mathbb{E}\left[
\left\|
\sum_{i=1}^N\alpha_i\nabla\mathcal{L}_i(\theta_i^t)-g^t
\right\|^2
\right]&\leq
\sum_{i=1}^N\alpha_i
\mathbb{E}\left[
\left\|
\nabla\mathcal{L}_i(\theta_i^t)-g_i^t
\right\|^2
\right]\\
&\leq
(1-\overline{\epsilon}^{\,t})
\overline{\pmb{\omega_D}}^2
+\overline{\epsilon}^{\,t}
\left(\overline{\pmb{\omega_{D,Q}}}\right)^2\\
&\leq
(1-\overline{\epsilon})
\overline{\pmb{\omega_D}}^2
+\overline{\epsilon}
\left(\overline{\pmb{\omega_{D,Q}}}\right)^2,
\end{align*}
where the last inequality follows from
$\overline{\pmb{\omega_{D,Q}}}
\geq\overline{\pmb{\omega_D}}$ and
$\overline{\epsilon}^{\,t}\leq\overline{\epsilon}$.
Substituting this bound into
Eq.~\eqref{eq:one_step_descent} gives
{\small \begin{align*}
a_{t+1}
&\leq
(1-\rho\eta)a_t
+\eta\left(
(1-\overline{\epsilon})
\overline{\pmb{\omega_D}}^2
+\overline{\epsilon}
\left(\overline{\pmb{\omega_{D,Q}}}\right)^2
-\frac{1}{2}\mathbb{E}[\|g^t\|^2]
\right)\\
&\quad+
L(1+\overline{\gamma})\eta^2
\left((1-\overline{\epsilon})
\overline{\pmb{\omega_D}}^2
+\overline{\epsilon}
\left(\overline{\pmb{\omega_{D,Q}}}\right)^2+G^2\right)\\
&\quad+
4L^2(R-1)^2G^2\eta^3
\left(
(1+\overline{\gamma})^2+\sigma_\gamma^2
\right).
\end{align*}}
Let $g_{\min}^2
:=
\min_{0\leq t<T}\mathbb{E}[\|g^t\|^2].$
Then
$a_{t+1}\leq(1-\rho\eta)a_t+C$
\begin{align*}
C:={}
&\eta\left(
(1-\overline{\epsilon})
\overline{\pmb{\omega_D}}^2
+\overline{\epsilon}
\left(\overline{\pmb{\omega_{D,Q}}}\right)^2-\frac{g_{\min}^2}{2}
\right)\\
&+
L(1+\overline{\gamma})\eta^2
\left((1-\overline{\epsilon})
\overline{\pmb{\omega_D}}^2
+\overline{\epsilon}
\left(\overline{\pmb{\omega_{D,Q}}}\right)^2+G^2\right)\\
&+
4L^2(R-1)^2G^2\eta^3
\left(
(1+\overline{\gamma})^2+\sigma_\gamma^2
\right).
\end{align*}
Since $0<\rho\eta<1$, iterating the recurrence gives
$a_T\leq
(1-\rho\eta)^Ta_0
+\frac{1-(1-\rho\eta)^T}{\rho\eta}\,C$
Therefore,
\begin{align*}
\mathbb{E}\left[
\mathcal{L}(\overline{\theta}^{T})
-\mathcal{L}(\theta^{\mathrm{PO}})
\right]&\leq
(1-\rho\eta)^T
\mathbb{E}\left[
\mathcal{L}(\overline{\theta}^{0})
-\mathcal{L}(\theta^{\mathrm{PO}})
\right]\\
&\quad+
\frac{1-(1-\rho\eta)^T}{\rho\eta}C
\end{align*}
\end{proof}

\subsection{Proof of Lemma~\ref{le:error df dfdtheta}}\label{pr:error df dfdtheta}
\begin{proof}
$\frac{\widehat{df}}{d\theta}$ is an estimate of
$\frac{df_i(\theta_t)}{d\theta}$.
In the remainder of the proof,
$\frac{df}{d\theta}$ denotes
$\frac{df_i(\theta_t)}{d\theta}$.
To simplify the notation, we denote
${f_i}_{z\sim D_i(\theta_i^t)}(z)$ as $f_t$ and
$\nabla {f_i}_{z\sim D_i(\theta_i^t)}(z)$ as $\nabla f_t$.
Similarly, ${f_i}_{z\sim Q_i}(z)$ is denoted as $f_{Q_i}$.
$\widehat f_t$ is the estimate of $f(\theta)$.
Because $\widehat f_t$ is computed from the effective sample set
$\widetilde{\mathcal S}_i^t$, whose residual contamination fraction is
$\epsilon_i^t$, and we estimate the mean of a Gaussian mixture or the
parameter $p$ of a binomial distribution
$X\sim\mathrm{Binomial}(n,p)$, we have
\begin{align*}
\widehat f_t
&=\widehat f_i(\widetilde{\mathcal S}_i^t)
=(1-\epsilon_i^t)
\widehat{{f_i}}_{z\sim D_i(\theta_i^t)}(z)
+\epsilon_i^t
\widehat{{f_i}}_{z\sim Q_i}(z)\\
&=(1-\epsilon_i^t)
\left(
{f_i}_{z\sim D_i(\theta_i^t)}(z)
+\mathrm{err}_i^t
\right)
+\epsilon_i^t f_{Q_i}
+\mathrm{err}_{i,\mathrm{sel}}^t.
\end{align*}
Here, $\mathrm{err}_{i,\mathrm{sel}}^t$ denotes the selection
error induced by the SVD-based robust-gradient filtering, including
the error caused by removing samples from the clean component
$D_i(\theta_i^t)$.
This decomposition applies to empirical mean and proportion
estimators; other estimators require an analogous contamination-error
decomposition or robust-estimation bound.
$\mathrm{err}_i^t$ is the sampling error of client $i$ at iteration $t$
and $\mathbb{E}[\mathrm{err}_i^t]=0$ for all $t\in[0,T-1]$.
For simplicity, denote $\theta_i^t$ as $\theta_t$,
$\theta_i^{t-1}$ as $\theta_{t-1}$,
$\mathrm{err}_i^t$ as $\mathrm{err}_t$, and
$\mathrm{err}_i^{t-1}$ as $\mathrm{err}_{t-1}$.
For each $1\leq j\leq d$, Taylor's theorem applied to
$f_{i,j}$ around $\theta_t$ implies that there exists
$\xi_{t-1,j}$ on the line segment between
$\theta_{t-1}$ and $\theta_t$ such that
$f_{t,j}-f_{t-1,j}
=\nabla f_{i,j}(\theta_t)^{\intercal}
(\theta_t-\theta_{t-1})
-\frac{1}{2}
(\theta_t-\theta_{t-1})^{\intercal}
\nabla^2 f_{i,j}(\xi_{t-1,j})
(\theta_t-\theta_{t-1})$.
The second term is the Taylor remainder and is retained explicitly in
the subsequent error analysis. Similarly, for
$\xi_{t-k,j}\in[\theta_{t-k},\theta_t]$, $1\leq k\leq H$, and
$H\geq d$, define $a_k=[a_{k,1},\ldots,a_{k,d}]^{\intercal}$, where
$a_{k,j}:=-\frac{1}{2}(\theta_t-\theta_{t-k})^{\intercal}
\nabla^2f_{i,j}(\xi_{t-k,j})(\theta_t-\theta_{t-k})$.
Then, $f_t-f_{t-k}=\frac{df}{d\theta}
(\theta_t-\theta_{t-k})+a_k$.
Let $\Delta_k f:=\widehat f_t-\widehat f_{t-k}$. Using the
effective-set decomposition above, we have
\begin{align*}
\Delta_k f
&=(1-\epsilon_i^t)\left(
\frac{df}{d\theta}(\theta_t-\theta_{t-k})+a_k+
\mathrm{err}_t-\mathrm{err}_{t-k}\right)\\
&\quad+\mathrm{err}_{i,\mathrm{sel}}^t
-\mathrm{err}_{i,\mathrm{sel}}^{t-k}
+(\epsilon_i^t-\epsilon_i^{t-k})
(f_{Q_i}-f_{t-k}-\mathrm{err}_{t-k}).
\end{align*}
Now we can obtain
\begin{align*}
\Delta f
&=\left[
\begin{array}{ccc}
\Delta_1f & \ldots & \Delta_Hf
\end{array}
\right],\\
\Delta\theta
&=\left[
\begin{array}{ccc}
\theta_t-\theta_{t-1} & \ldots & \theta_t-\theta_{t-H}
\end{array}
\right]\\
&=-
\left[
\begin{array}{c}
\nabla\mathcal{L}_{z\sim P_i(\theta_{t-1})}
(z,\theta_{t-1})\\
\vdots\\
\displaystyle\sum_{k=1}^{H}
\nabla\mathcal{L}_{z\sim P_i(\theta_{t-k})}
(z,\theta_{t-k})
\end{array}
\right]^{\intercal}\eta.
\end{align*}

The component-wise second-order Taylor remainders are collected into
the vectors $a_k$, and the corresponding cumulative remainder terms are
assembled into the matrix $A$. Similarly, the sampling-error differences
are assembled into the matrix $E$, defined as
$A:=[a_1\ \ldots\ a_H]$ and
$E:=[\mathrm{err}_t-\mathrm{err}_{t-1}\ \ldots\
\mathrm{err}_t-\mathrm{err}_{t-H}]$.
Define the effective-set error matrix as
$E_{\mathrm{sel},i}^t:=
[e_{\mathrm{sel},i,1}^t\ \ldots\
e_{\mathrm{sel},i,H}^t]$, where
$e_{\mathrm{sel},i,k}^t:=
\mathrm{err}_{i,\mathrm{sel}}^t-
\mathrm{err}_{i,\mathrm{sel}}^{t-k}
+(\epsilon_i^t-\epsilon_i^{t-k})
(f_{Q_i}-f_{t-k}-\mathrm{err}_{t-k})$.
Therefore,
$\Delta f
=(1-\epsilon_i^t)
\left(
\frac{df}{d\theta}\Delta\theta+A+E
\right)
+E_{\mathrm{sel},i}^t.$
With $\frac{\widehat{df}}{d\theta}:=
\Delta f(\Delta\theta)^\dagger$ and
$\operatorname{rank}(\Delta\theta)=d$, the Jacobian-estimation
error is
$\frac{\widehat{df}}{d\theta}-\frac{df}{d\theta}
=-\epsilon_i^t\frac{df}{d\theta}
+(1-\epsilon_i^t)A(\Delta\theta)^\dagger  +(1-\epsilon_i^t)E(\Delta\theta)^\dagger
+E_{\mathrm{sel},i}^t(\Delta\theta)^\dagger. $

$\|f_{Q_i}-f_{t'}\|_2
\leq W_1(D,Q)_{\max}$ for all $0\leq t'<T.$, and define
$(\omega_{\mathrm{sel},i})^2:=
\max_{0\leq t'<T}\mathbb{E}
[\|\mathrm{err}_{i,\mathrm{sel}}^{t'}\|^2]$. We use $f_\varphi(\varphi_i;n_i)$ to specify a uniform
clean-reference sampling-error bound
$\mathbb{E}[\|\mathrm{err}_i^{t'}\|_2^2]
\leq f_\varphi(\varphi_i;n_i)/n_i$.
Since $0\leq\epsilon_i^{t'}\leq1$ and
$\epsilon_i=\max_{0\leq t'<T}\mathbb{E}[\epsilon_i^{t'}]$,
we have $\mathbb{E}[(\epsilon_i^t)^2]\leq\epsilon_i$ and
$\mathbb{E}[|\epsilon_i^t-\epsilon_i^{t-k}|^2]
\leq2\epsilon_i$.
Using $\|x+y\|^2\leq2\|x\|^2+2\|y\|^2$, we obtain
$\mathbb{E}[\|E_{\mathrm{sel},i}^t\|_F^2]\leq
8H(\omega_{\mathrm{sel},i})^2+
8HW_1(D,Q)_{\max}^2\epsilon_i+
4Hf_\varphi(\varphi_i;n_i)/n_i$.

Under the uniform singular-value condition stated in
Lemma~\ref{le:error df dfdtheta},
$\|(\Delta\theta)^\dagger\|_2
\leq1/(\eta\lambda_{i,\min})$.
By $\|\sum_{j=1}^4x_j\|_F^2
\leq4\sum_{j=1}^4\|x_j\|_F^2$,
$(1-\epsilon_i^t)^2\leq1$, and
$\|df/d\theta\|_F\leq F$,
{\small
\begin{align}
&\mathbb{E}\left[
\left\|
\frac{\widehat{df}}{d\theta}-\frac{df}{d\theta}
\right\|_F^2\right]\notag\\
&\leq
4F^2\mathbb{E}[(\epsilon_i^t)^2]
+\frac{4}{\eta^2\lambda_{i,\min}^2}
\left(
\mathbb{E}[\|A\|_F^2]
+\mathbb{E}[\|E\|_F^2]
+\mathbb{E}[\|E_{\mathrm{sel},i}^t\|_F^2]
\right).
\label{eq:df3}
\end{align}
}
Here, $\|\cdot\|_F$ denotes the Frobenius norm.
The Taylor-remainder bound gives
$\mathbb{E}[\|A\|_F^2]
\leq M^2\eta^4G^4H^6/4$.
Without requiring independence across iterations,
$\mathbb{E}[\|E\|_F^2]
\leq4Hf_\varphi(\varphi_i;n_i)/n_i$.
Substituting these bounds gives
{\small
\begin{align}
&\mathbb{E}\left[
\left\|
\frac{\widehat{df}}{d\theta}-\frac{df}{d\theta}
\right\|_F^2\right]\notag\\
&\leq
4\epsilon_iF^2+
\frac{32HW_1(D,Q)_{\max}^2}{\eta^2}
\epsilon_i\lambda_{i,\min}^{-2}\notag\\
&\quad+
\left(
M^2\eta^2G^4H^6+
\frac{32Hf_\varphi(\varphi_i;n_i)/n_i+
32H(\omega_{\mathrm{sel},i})^2}{\eta^2}
\right)\lambda_{i,\min}^{-2}.
\label{eq:df5}
\end{align}
}

\end{proof}

\subsection{Proof of Lemma \ref{le:error L2D with h}}\label{proof: error L2D with h}
\begin{proof}
 By Eq.~\ref{eq:PG} we obtain that $\nabla \mathcal{L}_{i,2}(\theta_i^{t}) = \ell(Z_i;\theta_i^{t}){\frac{df_{i,t}}{d\theta} }^\intercal \frac{\partial \log p\left(Z_i;f_i(\theta_i^{t})\right)}{\partial f_i(\theta_i^{t})}$ and $\widehat{\nabla}\mathcal{L}_{{i,2} z\sim D_i(\theta_i^{t})}(z,\theta_i^{t}) = \widehat{\ell}(Z_i;\theta_i^{t}){\frac{\widehat{df_{i,t}}}{d\theta} }^\intercal\frac{\partial \log p\left(Z_i;\widehat{f}_i(\theta_i^{t})\right)}{\partial \widehat{f}_i(\theta_i^{t})}.$ For clean-reference estimators, as $n_i\to\infty$, $\mathbb{E}  \left[\left \| \ell(Z_i;\theta) -  \widehat{\ell}(Z_i;\theta)\right\|^2\right] $ and $\mathbb{E} \left[\left \|\frac{\partial \log p\left(Z_i;f_i(\theta)\right)}{\partial f_i(\theta)} - \frac{\partial \log p\left(Z_i;\widehat{f}_i(\theta)\right)}{\partial \widehat{f}_i(\theta)} \right\|^2\right]$ vanish. Data-dependent selection errors need not vanish. Thereby, $\mathbb{E} \left[
\left\|
\nabla\mathcal{L}_{i,2}(\theta_i^{t})
-
\widehat{\nabla}\mathcal{L}_{{i,2} z\sim D_i(\theta_i^{t})}(z,\theta_i^{t})
\right\|^2
\right]$ mainly consists of two components. The first part is the following:
{\small
\begin{align}
    &\mathbb{E} \left[\left \| \ell(Z_i;\theta_i^{t}) \right\|^2  \left \| \frac{df_{i,t}}{d\theta}-\frac{\widehat{df_{i,t}}}{d\theta} \right\|^2   \left \|  \frac{\partial \log p\left(Z_i;f_i(\theta_i^{t})\right)}{\partial f_i(\theta_i^{t})} \right\|^2 \right] \notag\\
    &\leq \ell^2_{\max}  \left \|  \frac{\partial \log p\left(Z_i;f_i(\theta_i^{t})\right)}{\partial f_i(\theta_i^{t})} \right\|^2_{\max}  \mathbb{E} \left[ \left \| \frac{df_{i,t}}{d\theta}-\frac{\widehat{df_{i,t}}}{d\theta} \right\|^2  \right] \label{eq:error L2 hat 1}\\
    & = \mathcal{O}\left(\ell^2_{\max} d\left(\pmb{\omega_F}\right) \right).\label{eq:error L2 hat 2}
\end{align}}
Because $\Theta$ is a closed and bounded set and a continuous function has a maximum value on $\Theta$, there is an upper bound $\ell^2_{\max}$ of $\left \| \ell(Z_i;\theta) \right\|^2$ and an upper bound $\left \|  \frac{\partial \log p\left(Z_i;f_i(\theta_i^{t})\right)}{\partial f_i(\theta_i^{t})} \right\|^2_{\max}$ of  $\left \|  \frac{\partial \log p\left(Z_i;f_i(\theta_i^{t})\right)}{\partial f_i(\theta_i^{t})} \right\|^2$. Eq.~(\ref{eq:error L2 hat 2}) arises from the upper bound of $ \left \|  \frac{\partial \log p\left(Z_i;f_i(\theta_i^{t})\right)}{\partial f_i(\theta_i^{t})} \right\|^2$ is $\mathcal{O}(d)$ and it is related to the covariance $\Sigma$ of the data for the Gaussian distribution.
The clean-reference sampling term vanishes as $n_i\to\infty$.
The upper bound of this part is
$\mathcal{O}\left(
\ell^2_{\max}d F^2
\frac{f_\varphi(\varphi_i;n_i)}{n_i}
\right)$
by using Lemma 12 in \citep{pmlr-v139-izzo21a}.
Then we can obtain
$\mathbb{E}\left[
\left\|
\nabla\mathcal{L}_{i,2}(\theta_i^{t})
-\widehat{\nabla}\mathcal{L}_{{i,2} z\sim D_i(\theta_i^{t})}
(z,\theta_i^{t})
\right\|^2
\right]
\leq
\mathcal{O}\left(
\ell^2_{\max}d
\left(
\pmb{\omega_F}
+F^2\frac{f_\varphi(\varphi_i;n_i)}{n_i}
\right)
\right).$
{\small
\begin{align}
&\mathbb{E}\left[\left \| {\nabla}\mathcal{L}_{{i,2}}(Q_i,\theta_i^{t}) - \widehat{\nabla}\mathcal{L}_{{i,2}}(Q_i,\theta_i^{t})\right\|^2 \right] \notag\\
    &= \mathbb{E}\left[\left \| {\ell}_{z\sim Q_i}(Z_i;\theta_i^{t}){\frac{\widehat{df_{i,t}}}{d\theta} }^\intercal\frac{\partial \log p\left(Z_i;\widehat{f}_i(\theta_i^{t})\right)}{\partial \widehat{f}_i(\theta_i^{t})} \right\|^2\right] \label{eq:hatL2Q1}\\
    &\leq  \mathcal{O}\left( \ell^2_{\max} d \mathbb{E} \left[\left \| \frac{\widehat{df_{i,t}}}{d\theta} \right\|^2 \right]\right) \label{eq:hatL2Q2}\\
     &\leq  \mathcal{O}\left( 2\ell^2_{\max} d \left(\mathbb{E} \left[\left \| \frac{{df_{i,t}}}{d\theta} \right\|^2 \right]+\mathbb{E} \left[\left \| \frac{df_{i,t}}{d\theta}-\frac{\widehat{df_{i,t}}}{d\theta} \right\|^2 \right] \right)\right)\label{eq:hatL2Q3}\\
    &= \mathcal{O}\left(2\ell^2_{\max} d \left(\left(\pmb{\omega_F}\right) + F^2\right) \right)\label{eq:hatL2Q4}.
\end{align}
}

Eq.~(\ref{eq:hatL2Q1}) arises from Eq.~\eqref{eq:PG} and ${\nabla}\mathcal{L}_{{i,2}}(Q_i,\theta_i^{t})=0$. (\ref{eq:hatL2Q2}) arises from Lemma~10 of \citep{pmlr-v139-izzo21a} and Lemma~\ref{le:upper bounds}. Inequality (\ref{eq:hatL2Q3}) arises from Jensen's inequality of two terms. Inequality \eqref{eq:hatL2Q4} arises from Lemma~\ref{le:error df dfdtheta}.
\end{proof}

\subsection{Additional Lemmas and Their Proofs}
\begin{lemma}\label{le:upper bound variance}
Let $\overline{\gamma}=\sum_{i=1}^N\alpha_i\gamma_i$ and
$\sigma_\gamma^2=\sum_{i=1}^N
\alpha_i(\gamma_i-\overline{\gamma})^2$.
Then $\mathbb{E}\left[
\sum_{i=1}^N\alpha_i(1+\gamma_i)^2
\left\|\overline{\theta}^{t}-\theta_i^{t}\right\|^2
\right]$ is upper bounded by $4(R-1)^2G^2\eta^2
\left((1+\overline{\gamma})^2+\sigma_\gamma^2\right).$
\end{lemma}

\begin{proof}[Proof of Lemma~\ref{le:upper bound variance}]
Let $t_0$ denote the most recent synchronization iteration and define
$u_i^t:=\eta\sum_{s=t_0}^{t-1}
\nabla\mathcal{L}_i(\theta_i^s)$, where
$t-t_0\leq R-1$. Then $\theta_i^t=\overline{\theta}^{t_0}-u_i^t,
\overline{\theta}^t
=\overline{\theta}^{t_0}
-\sum_{h=1}^N\alpha_hu_h^t.$
By Lemma~\ref{le:upper bounds},
$\|\eta\sum_{s=t_0}^{t-1}
\nabla\mathcal{L}_i(\theta_i^s)\|\leq(R-1)\eta G$. Hence,
\[
\left\|\overline{\theta}^t-\theta_i^t\right\|
\leq
\|u_i^t\|+
\left\|\sum_{h=1}^N\alpha_hu_h^t\right\|
\leq 2(R-1)\eta G.
\]
\begin{align*}
&\mathbb{E}\left[
\sum_{i=1}^N\alpha_i(1+\gamma_i)^2
\left\|\overline{\theta}^{t}-\theta_i^{t}\right\|^2
\right] \\
&\leq
4(R-1)^2G^2\eta^2
\sum_{i=1}^N\alpha_i(1+\gamma_i)^2 \\
&=
4(R-1)^2G^2\eta^2
\left((1+\overline{\gamma})^2+\sigma_\gamma^2\right),
\end{align*}.
\end{proof}

\begin{proof}[Proof of Lemma~\ref{le:upper bound L2 with h}]
Conditioning on $\epsilon_i^t$, the residual contamination
fraction is fixed. By the linearity of the empirical PG estimator,
$g_i^t
=(1-\epsilon_i^t)
\widehat{\nabla}\mathcal{L}_i
(D_i(\theta_i^t),\theta_i^t)+\epsilon_i^t
\widehat{\nabla}\mathcal{L}_i(Q_i,\theta_i^t).$
Because $Q_i$ is independent of $\theta$,
$\nabla\mathcal{L}_{i,2}(Q_i,\theta_i^t)=0$.
Assumption~\ref{as:smoothness} and the coupling characterization of
$W_1$ give
\[
\left\|
\nabla\mathcal{L}_{i,1}(\theta_i^t)
-\nabla\mathcal{L}_{i,1}(Q_i,\theta_i^t)
\right\|
\leq LW_1(D,Q)_{\max}.
\]
We assume the following effective-set component-error
bounds hold uniformly conditional on $\epsilon_i^t$,
including finite-sample and data-dependent selection errors:
\[
\mathbb{E}\left[
\left\|
\nabla\mathcal{L}_i(\theta_i^t)
-\widehat{\nabla}\mathcal{L}_i
(D_i(\theta_i^t),\theta_i^t)
\right\|^2
\,\middle|\,\epsilon_i^t
\right]
\leq\boldsymbol{\omega}_D^2
\]
\begin{align*}
&\mathbb{E}\left[
\left\|
\nabla\mathcal{L}_i(\theta_i^t)
-\widehat{\nabla}\mathcal{L}_i(Q_i,\theta_i^t)
\right\|^2
\,\middle|\,\epsilon_i^t
\right]\\
&\leq
\left(
LW_1(D,Q)_{\max}
+\boldsymbol{\omega}_Q
+\boldsymbol{\omega}_D
\right)^2.
\end{align*}
Here, $\boldsymbol{\omega}_D$ and $\boldsymbol{\omega}_Q$
include the finite-sample and data-dependent selection errors and are
uniform over the filtering outcomes.
By the convexity of the squared norm,
\begin{align*}
&\mathbb{E}\left[
\left\|
\nabla\mathcal{L}_i(\theta_i^t)-g_i^t
\right\|^2
\,\middle|\,\epsilon_i^t
\right]\\
&\leq
(1-\epsilon_i^t)\boldsymbol{\omega}_D^2
+\epsilon_i^t
\left(
LW_1(D,Q)_{\max}
+\boldsymbol{\omega}_Q
+\boldsymbol{\omega}_D
\right)^2.
\end{align*}
\end{proof}

\section{Case Studies}\label{app:case studies}
    \subsection{Pricing with Dynamic Demands} Consider a company that aims to find the proper prices for various goods. Let $\theta\in\mathbb{R}^d$ be the prices associated with the set of $d$ goods, and $Z_i\in\mathbb{R}^d$ be the retailer $i$'s demands for these goods. We assume $Z_i\sim \mathcal{N}(\mu(\theta),\sigma^2)$ is model dependent Gaussian distribution, i.e., mean demand for each good decreases linearly as the price increases. $\mu(\theta)= \mu_0-\gamma_i\theta$, here $\gamma_i\geq 0$ measures the price sensitivity and varies across different retailers. The goal is to maximizes total revenue over all retailers $\sum_{i=1}^{N} \alpha_i \mathbb{E}_{Z_i \sim D_i(\theta)} [\theta^T Z_i]$.
    \subsection{Pricing with Dynamic Contribution} Consider a similar setting as above where a company aims to find prices for various goods that maximize the total revenue. Suppose the retailers it faces have consumers from $K$ groups (e.g., old/young) and these consumers have fixed demands but they have options to purchase the goods from other companies. Here we assume the demands of retailer $i$'s consumers from group $k\in[K]$ have fixed distribution $D_{i,k}=\mathcal{N}(\mu_{i,k},\sigma_i^2)$, but the fraction of groups change dynamically based on prices $\theta$. Let $r_{i,k}(\theta) = B-\mathbb{E}_{Z_i\sim D_{i,k}}[\theta^T Z_i]$ be the expected remaining budget of group $k$ after purchasing the goods (with initial budget $B$) and assume the fraction of group $k$ is $\nu_{i,k}(\theta)=\frac{r_{i,k}(\theta)}{\sum_{k=1}^K r_{i,k}(\theta)}$ (i.e., groups with more remaining budgets have more incentives to stay in the system). Then the retailer $i$'s total demands would follow the mixture of distributions $D_i(\theta) = \sum_{k=1}^K\nu_{i,k}(\theta)D_{i,k}$.
     \subsection{Binary Strategic Classification} Consider an FL system is used to make binary decisions about human agents (e.g., loan application, hiring, college admission), where client $i$ with data $Z_i=(X_i,Y_i)$ may manipulate their features strategically to increasing their chances of receiving favorable outcomes at the lowest costs. Suppose the clients are assigned positive decisions whenever $1/\left(1+\exp(-X^T\theta)\right)\geq 1/2$ and the cost it takes to manipulate feature from $X_i$ to $X_i'$ is $\frac{1}{2\gamma_i}||X'-X||$, then strategic clients would react to the decision model $\theta$ by manipulating their features to $\arg\max_{X_i'}\left(- \langle X'_i, \theta \rangle - \frac{1}{2\gamma_i} \|X_i' - X_i\|_2^2 \right) = X_i - \gamma_i\theta$. For the initial data, we acquire it from a synthetic data where $X|Y=y\sim \mathcal{N}(\mu_y,\sigma_y^2)$ for $y\in\{0,1\}$ and two realistic \textit{figures/Adult.png} and \textit{Give Me Some Credit} data. The details of these data are in Appendix~\ref{app:dataset}. The goal is to minimizes prediction loss over all clients.
    \subsection{House Pricing Regression} Consider a regression task that aims to find the proper listing prices for houses based on features such as size, age, number of bedrooms. Suppose the features of houses in district $i$ follow $X_i\sim \mathcal{N}(\mu_i,\sigma_i^2)$. Given predictive model $\theta$, let the listing price be $h(X_i)=\theta^TX_i$. Since the houses with higher listing prices tend to lower the demand, the
    actual selling price $Y$ depends on $\theta$ and we assume $Y_i=(a_i-\gamma_i\theta)^TX_i+\epsilon$, where $\epsilon\sim \mathcal{N}(0,\sigma^2_{n})$ is Gaussian noise. The goal is to minimizes  prediction error
    \subsection{Regression with dynamic contribution} Consider an example of retail inventory management, where the regression model $\theta$ predicts product sales considering features like seasonal trends and advertising. Distributor $i$ has retailers from two groups $k\in\{1,2\}$, and each retailer's data distribution is fixed. The fraction of retailers from each group is $\nu_{i,k}(\theta) =\frac{\ell_{i,-k}(\theta)+c}{\sum_{k'\in\{1,2\}}\left(\ell_{i,k'}(\theta)+c\right)}$ where $-k=\{1,2\}\setminus k$ because the retailers are more likely to stop purchasing if they experience higher prediction error $\ell_{i,k}(\theta)$. The goal is to find $\theta$ that minimizes expected prediction error $\sum_{i=1}^{N} \alpha_i \mathbb{E}_{(X_i,Y_i) \sim D_i(\theta)} \left[(\theta^TX_i-Y_i)^2\right]$, where $Y_i$ is a weighted combination of fixed distributions.

\section{Training details}\label{Training details}
The experimental setup and parameter selections follow the methodology established in \citep{pmlr-v139-izzo21a,jin2023performative}, ensuring consistency and comparability with existing results.

\subsection{Realistic Data}\label{app:dataset}
\begin{itemize}
    \item \underline{\textit{figures/Adult.png}} data \citep{misc_adult_2} where the goal is to predict whether a person's annual income exceeds \$50K based on 14 features (e.g.,  age, years of education).
    \item \underline{\textit{Give Me Some Credit}} data \citep{GiveMeSomeCredit} which has 11 features (e.g., monthly income, debt ratio) and can be utilized to predict whether a person
has experienced a 90-day past-due delinquency or worse.
\end{itemize}

\subsection{Loss Functions and Ridge Penalty}
We use ridge-regularized cross-entropy loss for all binary classification cases, and the ridge penalty is $0.01$. We use ridge-regularized squared loss for all regression cases, and the ridge penalty is $\frac{10}{3}$.

\subsection{Parameters and Distributions}\label{app:parameters}
The experimental parameters are listed in Table \ref{tab:parameters}.

For the targeted label-flipping experiment in Table~\ref{tab:label_flipping_gradient}, we generate
a three-dimensional binary classification dataset containing 1000
samples, with equal numbers of positive and negative examples. The
positive and negative features are generated from Gaussian
distributions with means $(2,2,0.5)^T$ and $(-2,-2,-0.5)^T$,
respectively, and identity covariance matrices. We randomly flip $(10\%, 20\%, 30\%, 40\%)$
 of the positive labels from $1$ to $0$. The performative shift
strength is set to $0.5$, and the ridge regularization parameter is
$0.01$. For the SVD-based filtering procedure, we use
$B=50$, $J_1=0.01$, and $J_2=0.05$.

\begin{table}[!htbp]
    \centering
    \caption{Parameters Used in the Experiments}
    \fitwidth{0.95\textwidth}{!}{%
    \begin{tabular}{llllll}
        \toprule
        Figure/Table
        & $\eta$
        & $R$
        & $H$
        & $n_i$
        & $\epsilon_i$ \\
        \midrule

        Table I: \textsc{PG}
        & 0.01
        & $-$
        & 1
        & 500
        & 0 \\

        Table I: \textsc{ProFL}
        & 0.01
        & 5
        & 1
        & 500
        & 0 \\

        \midrule

        Table II
        & 0.001
        & 5
        & 25
        & 500
        & 0, 0.1, 0.2, 0.3 \\

        Table III: Label flipping
        & $-$
        & $-$
        & $-$
        & 1000
        & 0.1, 0.2, 0.3, 0.4\\

        Table IV
        & 0.03
        & 5
        & 1
        & 2368
        & 0 \\

        \midrule

        Figure 2(b)
        & 0.0001
        & 3
        & 10
        & 60
        & 0 \\

        Figure 2(c)
        & 0.001
        & 5
        & 100
        & 1000
        & 0 \\

        Figure 3(a)
        & 0.001
        & 5
        & 20
        & 50, 500, 5000
        & 0 \\

        Figure 3(b)
        & 0.01, 0.03, 0.001
        & 5
        & 20
        & 2500
        & 0 \\

        Figure 3(c)
        & 0.001
        & 5
        & 4, 20, 100
        & 500
        & 0 \\

        Figure 4(a)
        & 0.03
        & 10
        & 500
        & 500
        & 0 \\

        Figure 4(b)
        & 0.03
        & 500
        & 1
        & 500
        & 0 \\

        Figure 4(c)
        & 0.01
        & 5
        & 1
        & 5812
        & 0 \\

        Figure 4(d)
        & 0.03
        & 5
        & 1
        & 2368
        & 0 \\

        Figure 5(a)--5(b)
        & \textsc{PFL}: 0.005; \textsc{ProFL}: 0.001
        & 5
        & 1
        & 500
        & 0 \\

        Figure 6(a)
        & \textsc{PFL}: 0.005; \textsc{ProFL}: 0.001
        & 5
        & 0.2, 0.5, 1
        & 500
        & 0 \\

        Figure 6(b)
        & \textsc{PFL}: 0.005; \textsc{ProFL}: 0.001
        & 5
        & 1
        & 5, 50, 500
        & 0 \\

        \bottomrule
    \end{tabular}%
    }
    \label{tab:parameters}
\end{table}

\subsection{Experiments Compute Resources}\label{app:cpu}
CPU: 11th Gen Intel(R) Core(TM) i7-11800H @ 2.30GHz   2.30 GHz.

\FloatBarrier
\bibliographystyle{apalike}
\bibliography{example_paper}

\end{document}